\documentclass[11pt,natbib]{article}
\usepackage{geometry}                
\usepackage[utf8]{inputenc}
\usepackage[english]{babel}

\usepackage{filecontents}
\usepackage{amsmath}
\usepackage{amssymb}
\usepackage{amsthm}
\usepackage{xcolor}
\usepackage{authblk}
\usepackage{subfiles}

\usepackage{breqn}
\usepackage{scalerel}
\usepackage{enumitem}
\usepackage{mathtools}
\newtheorem{theorem}{Theorem}

\usepackage{float}
\usepackage{graphicx}
\graphicspath{{./Figures/}}

\allowdisplaybreaks

\usepackage{natbib}

\usepackage{graphicx}
\graphicspath{{./Figures/}}

\usepackage{url}
\usepackage{etoolbox}
\gappto{\UrlBreaks}{\UrlOrds}

\usepackage{xr}
\makeatletter
\newcommand*{\addFileDependency}[1]{
  \typeout{(#1)}
  \@addtofilelist{#1}
  \IfFileExists{#1}{}{\typeout{No file #1.}}
}
\makeatother
\newcommand*{\myexternaldocument}[1]{%
    \externaldocument{#1}%
    \addFileDependency{#1.tex}%
    \addFileDependency{#1.aux}%
}
\myexternaldocument{arxiv_supplement}

\title{Model updating after interventions paradoxically introduces bias}

\author[1,2,*]{James~Liley}
\author[3]{Samuel~R.~Emerson}
\author[2,4,5]{Bilal~A.~Mateen}
\author[1,2,*]{Catalina~A.~Vallejos}
\author[2,3,*]{Louis~J.~M.~Aslett}
\author[2,6,*]{Sebastian~J.~Vollmer}
\affil[1]{MRC Human Genetics Unit, University of Edinburgh, UK}
\affil[2]{The Alan Turing Institute, London, UK}
\affil[3]{Department of Mathematical Sciences, Durham University, UK}
\affil[4]{Kings College Hospital, London, UK}
\affil[5]{Wellcome Trust, London, UK}
\affil[6]{Warwick Mathematics Institute, University of Warwick, UK}
\affil[*]{Co-corresponding authors}

\begin{document}
\maketitle

\begin{abstract}

Machine learning is increasingly being used to generate prediction models for use in a number of real-world settings, from credit risk assessment to clinical decision support. Recent discussions have highlighted potential problems in the updating of a predictive score for a binary outcome when an existing predictive score forms part of the standard workflow, driving interventions. In this setting, the existing score induces an additional causative pathway which leads to miscalibration when the original score is replaced. We propose a general causal framework to describe and address this problem, and demonstrate an equivalent formulation as a partially observed Markov decision process. We use this model to demonstrate the impact of such `naive updating' when performed repeatedly. Namely, we show that successive predictive scores may converge to a point where they predict their own effect, or may eventually tend toward a stable oscillation between two values, and we argue that neither outcome is desirable. Furthermore, we demonstrate that even if model-fitting procedures improve, actual performance may worsen. We complement these findings with a discussion of several potential routes to overcome these issues.

\vspace{5pt}
\emph{Note: Sections of this preprint on `Successive adjuvancy' (section~\ref{sec:successive_adjuvancy}, theorem~\ref{thm:successive_adjuvancy}, figures~\ref{fig:chaos},~\ref{fig:causality_sa}, and associated discussions) were not included in the originally submitted version of this paper due to length. This material does not appear in the published version of this manuscript, and the reader should be aware that these sections did not undergo peer review.}
\end{abstract}

\clearpage

\section{Introduction}

A common machine learning task concerns the prediction of an outcome $Y$ given a known set of predictors $X$~\citep{friedman01}. Usually, the intent is to anticipate the value of $Y$ in situations in which only $X$ is known. Often, the ultimate goal is to avoid or encourage certain values of $Y$, with interventions guided by the predictions provided by the algorithm.

We focus on the standard setting, often seen in healthcare, where $X$ is first observed and used to make predictions about $Y$, then interventions occur before outcomes are observed.
This setting can lead to prediction scores being `victims of their own success'~\citep{lenert19,sperrin19}. Interventions driven by the score can change the distribution of the data and outcomes, leading to a decay in observed performance, particularly if the intervention is successful. Analysis of this effect requires consideration of the causal processes governing $X$, $Y$, and the potential interventions driven by the  score~\citep{sperrin19}. Predictive scores are often implemented by direct dissemination to agents that are capable of modifying these causal processes~\citep{rahimian18,hyland20}, which leads to vulnerability to this problem. This problem also exist if predictions influence discrete actions, initial progress for this has been made using bandits \citep{Shi2020-na}. The phenomenon in which a predictive model influences its own effect has been called `performative prediction'~\citep{perdomo20}, and is of interest in model fairness~\citep{liu18,elzayn19}, in that actions taken in response to a model may pervert fairness metrics under which the model was designed.

This problem is particularly critical in settings where existing predictive scores are to be replaced by an updated version. In many real-world contexts, the underlying phenomena represented by the predictive model will change over time~\citep{wallace14}; statistical procedures for prediction may also improve (particularly for complex tasks); and researchers may wish to include further predictors or increase the scope of predictive scores. In general, we may expect that most predictive algorithms will need to be updated or replaced over time. Up-to-date models should generally be trained on the most recent available data which, as described above, will be contaminated by interventions based on existing scores. Should a new predictive model be fitted to new observations of $X$ and $Y$, it will consequently also model the impact of the existing score. Removal of the existing score will introduce bias into predictions made by the new score, as will insertion of the new score in place of the old. We term such an operation a `naive model replacement'. 

Our main aim is to introduce a general causal framework under which this phenomenon can be quantitatively studied. We use this framework to draw attention to the hazards of naive model replacement, especially when it occurs repeatedly. We introduce these hazards in the context of a generalised ultimate aim of the model, formulated as a constrained optimisation problem in which the occurrence of undesirable values of $Y$ is to be minimised with limited intervention. 
We also use our model to describe a second replacement strategy, `successive adjuvancy', in which new predictive scores are `added' to previous scores, with different emergent properties. 

A simple parable of this phenomenon concerns yearly influenza vaccinations. In a vaccination-naive population, risk assessments for influenza motivate widespread vaccination. However, in a later `epoch', the risk may appear much lower, and could naively suggest vaccination is no longer required introducing risks to public health\footnote{See for example \url{https://www.who.int/news-room/spotlight/ten-threats-to-global-health-in-2019}}. 
More generally, updated risk scores for clinical outcomes may be biased due to the interventions motivated by the scores themselves. 
As a second example, consider risk scores used to predict future emergency hospital admissions $Y$, on the basis of covariates $X$~\citep{rahimian18}. Suppose that prescription of some drug $D \in X$ confers increased risk, and this is established by the risk score. Should such risk scores be distributed at time $t=0$ to agents able to modify these factors (e.g., doctors), they may intervene by taking patients off $D$ thereby reducing emergency admission risk $\mathbb{E}[Y]$ at a time $t=1$. If a new score is naively fitted to $X$ at $t=0$ and $Y$ at $t=1$, it would underestimate the danger of $D$. 

Section~\ref{sec:model} describes the problem in terms of causal effects. We develop this into a full model specification in Section~\ref{sec:general}, along with a description of the constrained optimisation problem the model/intervention pair aims to solve in~\ref{sec:aim}. In Section~\ref{sec:naiveupdating}, we analyse the short and long-term effects of repeated naive replacement and show that they are generally undesirable
, and in section~\ref{sec:successive_adjuvancy} we describe  successive adjuvancy and examine long-term effects in a simplified setting. 
In Section~\ref{sec:solution}, we discuss three classes of solutions: more complex modelling, routine maintenance of a `hold-out' set, and controlled interventions. In Section~\ref{sec:control} we describe a reformulation of the model as control theory problem. Finally, in Section~\ref{sec:discussion}, we discuss limitations and implications of our approach. Our supplementary material contains relevant examples and proofs, an exposition of the problem in a real-world example, and a list of open problems in this setting.

\section{Model}
\label{sec:model}

\subsection{Overview}
\label{sec:overview}

Assume that we are attempting to predict an outcome $Y$ given a known set of covariates $X$. For simplicity, we assume $Y$ is a binary (e.g.~admission versus non admission to an Intensive Care Unit) and model it as a Bernoulli random variable. If $Y = 1$ is considered to be a negative outcome, often the eventual aim is to reduce $\mathbb{P}(Y = 1 | X) = \mathbb{E}[Y|X]$; we will discuss this in Section~\ref{sec:general} once we have defined terms formally. For the moment, we assume the causal structure shown in Figure~\ref{fig:causality}. We denote by $\rho_0(X)$ an initial predictive model for $\mathbb{E}[Y|X]$, fitted to observations of $(X,Y)$ generated under the causal structure in Figure~\ref{fig:causality}A. During deployment, we compute $\rho_0(X)$ for all members of a population and disseminate it to \textit{agents who can intervene} on $X$ (e.g.~doctors) based on those predictions, aiming to prevent $Y = 1$. Replacing or updating $\rho_0$, will typically involve fitting a new predictive model $\rho_1(X)$ to new observations of $(X,Y)$. It is clear that while $\rho_0(X)$ is an estimator of $\mathbb{E}[Y|X]$, the new predictive function $\rho_1(X)$ is instead an estimator of
\begin{equation}
\mathbb{E}\left[Y|X,\textrm{do}\left[\rho_0(X)\right]\right]  \label{eq:f1quantity}
\end{equation}
where $\textrm{do}\left[\rho_0(X)\right]$ indicates the action `compute and disseminate $\rho_0(X)$'. Although $\rho_0(X)$ is determined by $X$, the computation $\textrm{do}\left[\rho_0(X)\right]$ makes $\rho_0$ actionable. This opens a second causal pathway from $X$ to $Y$, affecting the setting in which $\rho_1$ is fitted (Figure~\ref{fig:causality}B). If the initial score $\rho_0(X)$ is universally disseminated, the distribution of $Y$ given $X$ (without the $\textrm{do}\left[\rho_0(X)\right]$) now becomes a counterfactual which we cannot observe. 

\begin{figure}[h]
\centering
\includegraphics[width=0.5\textwidth]{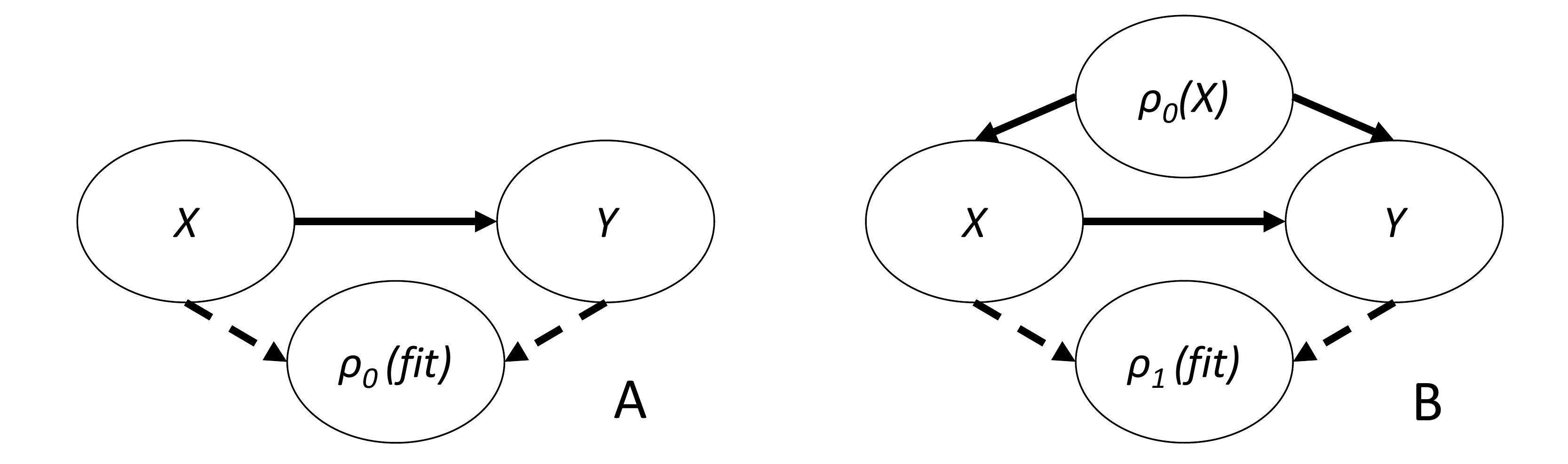}
\caption{Causal structure under which $\rho_0$ (panel A) and $\rho_1$ (panel B) are fitted. Dashed lines indicate a model-fitting process.}
\label{fig:causality}
\end{figure}

\subsection{General notation and assumptions}
\label{sec:general}

Here, we use a causal model to illustrate potential emergent behaviour resulting from repeated naive model updating, expanding out the `do'-operator used in section~\ref{sec:overview}. We do not aim to cover the complexities of \emph{all} real-world applications, yet our simplified setup is sufficient to demonstrate the dangers arising in this context. 

As $\rho_0$ is deployed and drives interventions, covariate values $X$ may change, as may the dependence of $Y$ on $X$. Here, we partition $X$ into three sets: 
\begin{align}
X^s &\textrm{: Fixed or `set' covariates; $\textrm{dim}(X^s)=p^s$}, \nonumber \\
X^a &\textrm{: Actionable covariates; $\textrm{dim}(X^a)=p^a$}, \nonumber \\
X^{\ell} &\textrm{: Latent covariates; $\textrm{dim}(X^{\ell})=p^{\ell}$}. \label{eq:saldef}
\end{align} 
Although $X^{\ell}$ may influence the causal mechanism between $X$ and $Y$ and may be intervened on, we assume it is unobserved. Hence, only $X^s$ and $X^a$ are known when evaluating a risk score, and $X^s$ cannot be intervened on (e.g.~`Age'). We also define two sets of time indicators $t,e$ (time, epoch): 
\begin{align}
t \in \{0,1\}: &\begin{cases} t=0\textrm{: predictive score is computed} \\ 
t=1\textrm{: $Y$ observed, after possible} \nonumber \\
\phantom{t=1:}\textrm{intervention} \end{cases} \nonumber \\
e \in \mathbb{N}: &\begin{cases} e=0\textrm{: no predictive score is used} \nonumber \\ e>0\textrm{: model from epoch $e-1$ is used.}  \end{cases}  
\end{align}
We assume that values of $X$ depend on $t$ and $e$ using the notation $X_e(t)=(X^s_e(t),X^a_e(t),X^{\ell}_e(t))\in \Omega^s \times \Omega^a \times \Omega^\ell = \Omega$. As $Y$ is only observed at $t=1$, $Y$ at epoch $e$ is denoted as $Y_e$. 
At each epoch, we assume that values of $X_e(t)$ across individuals in the population are $iid$ with probability measure $\mu_e$. We introduce the following functions
\begin{align}
f_e(x^s,x^a,x^{\ell}) &= \mathbb{E}\left[Y_e|X_e(1)=(x^s,x^a,x^{\ell})\right] \nonumber \\
&= \textrm{Causal mechanism determining} \nonumber \\
&\phantom{=} \textrm{probability of $Y_e = 1$ given $X_e(1)$} \nonumber \\
g^{a}_e(\rho,x^a) &\in \{g:[0,1] \times \Omega^a \to \Omega^a\} \nonumber \\
&=\textrm{Intervention process on $X^a$ in } \nonumber \\
&\phantom{=} \textrm{response to a predictive score $\rho$} \nonumber \\
&\phantom{=} \textrm{updating } X^a_e(0)\to X^a_e(1)\nonumber\\
g^{\ell}_e(\rho,x^{\ell}) &\in \{g:[0,1] \times \Omega^{\ell} \to \Omega^{\ell}\} \nonumber \\
&=\textrm{Intervention process on $X^{\ell}$ in } \nonumber \\
&\phantom{=} \textrm{response to a predictive score $\rho$} \nonumber \\
&\phantom{=} \textrm{updating } X^\ell_e(0)\to X^\ell_e(1)\nonumber\\
\rho_e(x^s,x^a) &\in \{\rho:\Omega^s \times \Omega^a \to [0,1]\} \nonumber \\
&= \textrm{Predictive score trained at epoch} \nonumber \\
&\phantom{=} \textrm{$e$, evaluated at observed covariates.} \nonumber 
\end{align}

Our main model is based on the following assumptions
\begin{enumerate}
\item $\forall e \hspace{5pt} X^s_e(0)=X^s_e(1)$: `set' covariates do not change from $t=0$ to $t=1$ \label{asm:first_main_assumption}
\item $X^a_0(0)=X^a_0(1)$, $X^{\ell}_0(0)=X^{\ell}_0(1)$: `actionable' and `latent' covariates do not change at epoch 0
\item $X^{\ell}_e(t)$ is unobserved, but may be modified from $t=0$ to $t=1$ in response to $\rho_{e-1}$
\item Values of $X_{e}(0)$ are independent across epochs, 
i.e. we do not track the same subjects over time. \label{asm:ident_dist}
\item At epoch $e$, the predictive score uses only $X^a_e(0)$,  $X^s_e(0)$ and $Y_e$ as training data; previous epochs are ignored and $X^a_e(1)$, $X^s_e(1)$ are not observed.  \label{asm:fourth_main_assumption}
\item $\forall e \hspace{5pt} \mathbb{E}[Y_e|X_e]=\mathbb{E}[Y_e|X_e(1)]$: $Y_e$ depends only on $X_e(1)$; that is, after any potential interventions. \label{asm:last_main_assumption}
\end{enumerate}
Besides these core assumptions, for the applications in this work, we variably assume some of the following
\begin{enumerate}[resume]
\item $f_e$, $g^a_e$, $g^{\ell}_e$ and $\mu_e$ remain fixed across epochs\footnote{In practice, we may assume $f_e$ changes slightly between epochs, but that this change is negligible.}, so values $\{ X^s_{\cdot}\}$ are \emph{iid}, as are $\{ X^a_{\cdot}\}$ and $\{ X^{\ell}_{\cdot}\}$ (within an epoch they may be correlated). Where we make this assumption, we will omit the epoch subscript for clarity. We also use the shorthand  $X^{\ell} \equiv X^{\ell}_e(0)|(X^s_e(0),X^a_e(0))=(x^s,x^a)$ \label{asm:equally_distributed} 

\item We allow $\rho_e$ to be an arbitrary function, but generally presume it is an estimator of
\begin{align}
&\rho_e(x^s,x^a) \approx \mathbb{E}\left[Y_e|X^s_e(0)=x^s,X^a_e(0)=x^a\right] \nonumber \\
&=\mathbb{E}_{X^{\ell}} \left[f_e\left(x^s,g^a_e(\rho_{e-1},x^a),g^{\ell}_e(\rho_{e-1},X^{\ell})\right)\right] \nonumber \\
&\triangleq \tilde{f}_e(x^s,x^a) \label{eq:rho_oracle} 
\end{align}
noting that $\tilde{f}_e$ depends on $e$ even if $f_e$ does not.
\item The function $f_e$ is $C^1$ in all arguments, and covariates are coded such that increases in covariate values increase risk \label{asm:fderiv}
\item $g^{\ell}_e$, $g^a_e$ are $C^1$ in all arguments, and a higher value of $\rho$ means a larger intervention is made (we assume $g^{\ell}_e$ and $g^a_e$ to be deterministic, but random valued functions may more accurately capture the uncertainty linked to real-world interventions).\label{asm:gderiv}
\end{enumerate}

This extended causal model is shown in Figure~\ref{fig:diagram_setup}. To aid interpretation, a real-world example is described using this notation in Supplementary Section~\ref{supp_sec:realistic_exposition}.

\subsection{Aim of predictive score}
\label{sec:aim}

The aim of the predictive score is generally to estimate $\mathbb{E}[Y_e|X_e(0)]$ accurately, presuming that we take $X_e(0)$ to be identically distributed over the population concerned. However, if action is to be taken on the score, we may presume the ultimate goal is to minimise $\mathbb{E}[Y_e]$, i.e. minimising
\begin{align}
&\mathbb{E}\left[Y_e\right] = \mathbb{E}_{X_e(0)}\left[ Y_e|X_e(1) \right] \nonumber \\ 
&=  \mathbb{E}_{X_e(0)}\left[f_e(X^s,g^a_e(\rho,X^a_e(0)),g^{\ell}_e(\rho,X^{\ell}_e(0)))\right] \label{eq:minimisethis}
\end{align}

However, we presume that we cannot afford to maximally intervene in all cases. Suppose the cost of lowering $X^a$ and $X^{\ell}$ by $x$ is $c^a(X^a,x)$ and $c^{\ell}(X^{\ell},x)$, respectively. The total intervention must then satisfy
\newpage
\begin{align}
&\mathbb{E}_{X_e(0)}\left[c^a\Big(X^a_e(0),X^a_e(0)-g^a_e(\rho,X^a_e(0))\Big) + \right. \nonumber \\
&\phantom{\mathbb{E}_{X_e(0)}} \left.  c^{\ell}\Big(X^{\ell}_e(0),X^{\ell}_e(0)-g^{\ell}_e(\rho,X^{\ell}_e(0))\Big)\right] \leq C \label{eq:subjecttothis}
\end{align}
for a known constant $C$, representing maximum cost. Thus we want to minimise~\eqref{eq:minimisethis} subject to~\eqref{eq:subjecttothis}. We have allowed $f_e$, $\mu_e$, $g^a_e$, $g^{\ell}_e$ and $\rho_e$ to vary across epochs. Of these, we can consider $f_e$ and $\mu_e$ to vary as a consequence of underlying processes, and $g^a_e$, $g^{\ell}_e$ and $\rho_e$ to be (somewhat) under our control.
Depending on the problem, we may either
consider $g^a_e$ and $g^{\ell}_e$ as fixed, and choose an optimal function $\rho_e$; or consider $\rho_e$ as fixed, and choose optimal functions $g^a_e$, $g^{\ell}_e$. If both are optimised, this corresponds to a general problem of resource allocation; see Supplementary Section~\ref{supp_sec:optimiseboth}.

\section{Naive model updating}
\label{sec:naiveupdating}

We consider a `naive' process in which a new score $\rho_e$ is fitted in each epoch, and then used as a drop-in replacement of an existing score $\rho_{e-1}$. We show that this procedure does not generally solve the constrained optimisation problem in Section~\ref{sec:aim}, can lead to `worse' performance of `better' models, and may lead to wide oscillation of predictions for fixed inputs across epochs.

\subsection{Worse performance of better models}

Here, we show that naive updating can lead to a loss in observed performance --- even when the procedure to infer $\rho_e$ is more accurate. We adopt assumptions~\ref{asm:first_main_assumption}--\ref{asm:gderiv}, taking the approximation in equation~\eqref{eq:rho_oracle} to be imperfect. Although most model elements are conserved across epochs (assumption~\ref{asm:equally_distributed}), we presume that the procedure used to infer $\rho_{e}$ changes, leading to better estimators of the function $\tilde{f}_e$. 

At epoch $e$, the training data is denoted by $(X_e^\star,Y_e^\star)$ and consists of $n$ samples of $(X_e(0),Y_e)$, with the latent covariate information removed.
In the absence of interventions, we assert that model performance will improve over epochs. 
Since performance under non-intervention is equivalent to performance at epoch 0, this can be stated as:
\begin{align}
&\mathbb{E}_{(X_0^\star,Y_0^\star)}\left[m_{\tilde{f}_0}(\rho_{e}|X_0^\star,Y_0^\star)\right] > \nonumber \\ 
&\mathbb{E}_{(X_0^\star,Y_0^\star)}\left[m_{\tilde{f}_0}(\rho_{e+1}|X_0^\star,Y_0^\star)\right], \label{eq:true_expectation2}
\end{align}
%

\begin{figure}[p]
\centering
\includegraphics[width=0.5\textwidth]{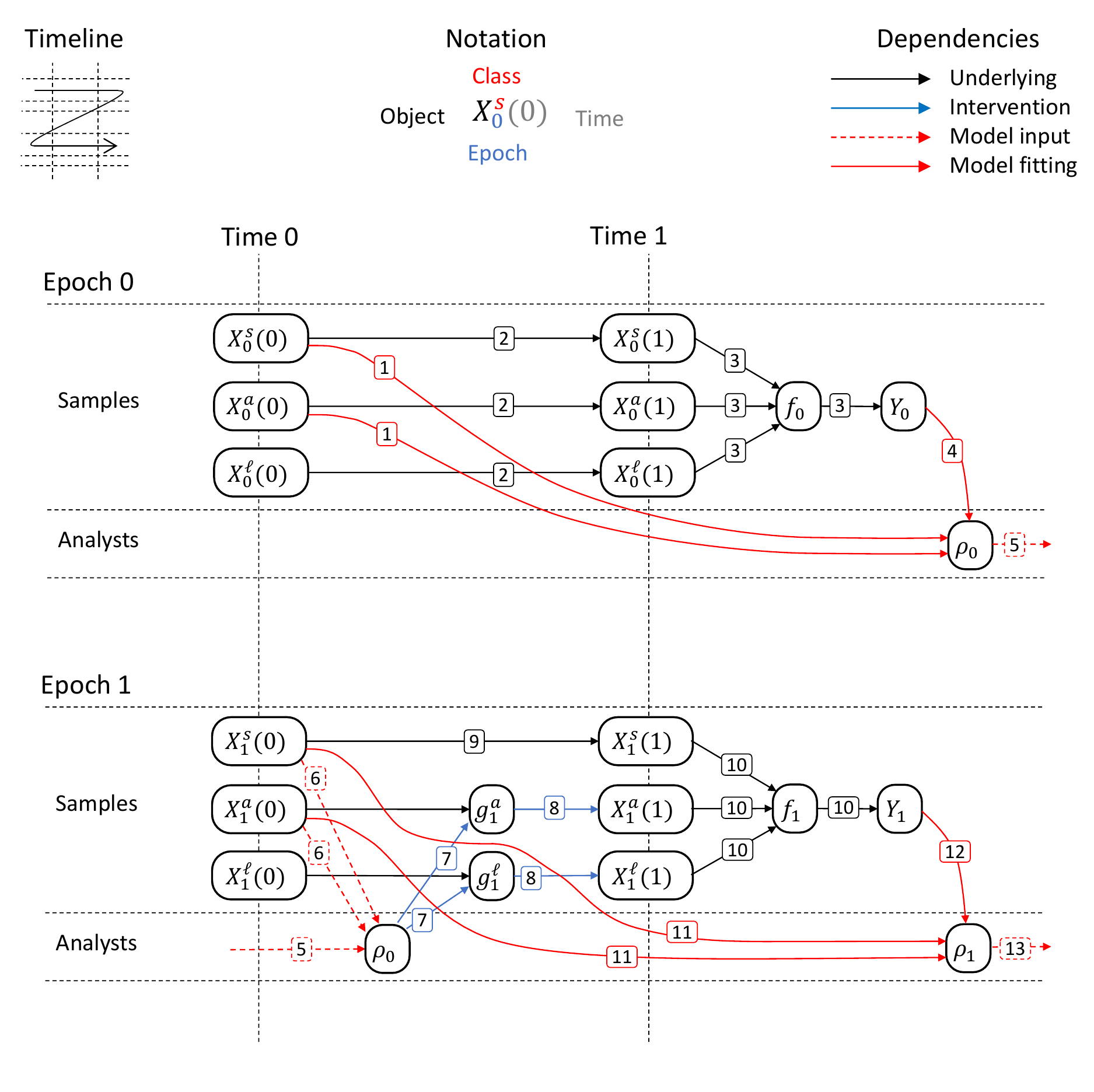}
\caption{This figure shows a causal diagram. An `epoch' is a new model fitting cycle. Covariates for a sample at the start of an epoch are modelled by $X^{\cdot}_e(0)$. We presume $\left\{X^s_e(0), e\geq 0\right\}$ are independent (as are $X^a_{\cdot}(0)$ and $X^{\ell}_{\cdot}(0)$).
We start with a sample at $t=0,e=0.$ 
The values $X^s_0(0)$, $X^a_0(0)$ are observed and sent to analysts (arrow 1). 
No predictive score is present and no interventions  are made based on it, so values remain the same to $t=1$ (arrows 2). 
$\mathbb{E}[Y_0]$ depends only on covariates at $t=1$, through $f_0$ (arrows 3). 
$Y_0$ is observed and sent to analysts (arrow 4)
who decide a function $\rho_{0}$, which is retained into epoch 1 (arrow 5). 
We start epoch 1 with a new independent sample. 
At $t=0,$ we observe $X^s_1(0)$, $X^a_1(0)$ and send them to analysts (arrow 6)
who compute $\rho_{0} \left( X^s_1(0),X^a_1(0)\right)$ which is used to inform interventions $g^a_1$, $g^{\ell}_1$
(arrow 7) 
to change values $X^a_e(0), X^{\ell}_e(0)$ to $X^a_e(1), X^{\ell}_e(1)$
respectively (arrows 8). 
$X^s_e(0)$ is not interventionable and becomes $X^s_e(1)$ (arrow
9). 
$\mathbb{E}[Y_1]$ is determined by covariates at $t=1$ (arrows
10). 
Analysts use the values of $X^s_1(0)$, $X^a_1(0)$ (arrows 11),
and $Y_1$ (arrow 12)
to decide a $\rho_{1}$, which is retained (arrow 13) for epoch 2.
Subsequent epochs proceed similarly to epoch 1.
}
\label{fig:diagram_setup}
\end{figure}

where $m_{\tilde{f}}(\rho | X,Y)$ denotes a metric for closeness of $\rho$ to $\tilde{f}$, given observed data $(X,Y)$\footnote{In practice, $m_{\tilde{f}_e}$ is unknown but (assuming latent covariates have a small influence on $f$) estimates of $m_{\tilde{f}_0}$ can be calculated through a holdout test data set.}. However, if interventions are in place, the improvement in equation~\eqref{eq:true_expectation2}, does not imply that the actual 
performance improves across epochs, that is: 
\newpage
\begin{align}
&\mathbb{E}_{(X_e^\star,Y_e^\star)}\left[m_{\tilde{f}_e}(\rho_{e}|X_e^\star,Y_e^\star)\right] \not> \nonumber \\ 
&\mathbb{E}_{(X_{e+1}^\star,Y_{e+1}^\star)}\left[m_{\tilde{f}_{e+1}}(\rho_{e+1}|X_{e+1}^\star,Y_{e+1}^\star)\right]. \label{eq:false_expectation2}
\end{align}
This is proved by counterexample: see Supplementary Section~\ref{supp_sec:models_worse}. A critical consequence of this artefact is that stakeholders may decide not to update an existing score, even if an apparently better one is available.\footnote{We note that practically (if a holdout test data set was used) the conclusions on performance made by stakeholders would be based on a risk score's closeness to $\tilde{f}_0$ instead of $\tilde{f}_e$, but the results are the same, which we show in Supplementary Section~\ref{supp_sec:models_worse}.}

\subsection{Dynamics of repeated naive updating}

Here, we analyse the dynamics of repeated naive model updating. For this purpose, we make assumptions~\ref{asm:first_main_assumption}-\ref{asm:gderiv} and assume that $\rho_e$ is an oracle: the `$\approx$' in equation~\eqref{eq:rho_oracle} is replaced by an `$=$'.

At epoch 0, there are no interventions, hence the risk of observing $Y = 1$ is $\mathbb{E}[Y_0|X_0(0)=(x^s,x^a,x^{\ell})] = f(x^s,x^a,x^{\ell})$. The score $\rho_0$ is therefore defined as
\begin{equation}
\rho_0(x^s,x^a)=\mathbb{E}_{X^{\ell}}[f(x^s,x^a,X^{\ell})], \label{eq:rho0def}
\end{equation}
where $X^{\ell}$ is denoted as in assumption~\ref{asm:equally_distributed}. In subsequent epochs, $\rho_e$ is used to modify $x^a$ and $x^{\ell}$ via $g^a$ and $g^{\ell}$, leading to the following recursive relation: 
\begin{align}
\rho_0(x^s,x^a) &= \mathbb{E}_{X^{\ell}}[f(x^s,x^a,X^{\ell})] \nonumber \\
\rho_e(x^s,x^a) &= \mathbb{E}_{X^{\ell}}[f(x^s,g^a(\rho_{e-1}(x^s,x^a),x^a), \phantom{)]} \nonumber \\
&\phantom{=\mathbb{E}_{X^{\ell}}[f(x^s,} g^{\ell}(\rho_{e-1}(x^s,x^a),X^{\ell}))] \nonumber \\
&\triangleq h(\rho_{e-1}(x^s,x^a)) \label{eq:hdef}
\end{align}
We briefly explore the dynamics of this recursion. Let $z \in [0,1]$ be arbitrary and denote by $S$ the substitution $(x^s,x^a,x^l)=\left(x^s,g^a(z,x^a),g^\ell(z,X^{\ell})\right)$. Recalling definitions of $p^s$, $p^a$ from~\eqref{eq:saldef}, we set (for $i$ across the dimensions of $(x^a, x^{\ell})$)
\begin{align}
\delta^{g^a}_i &= \frac{\partial [g^{a}(z,x^a)]_i}{\partial z}  \hspace{20pt} 
&\delta^{g^{\ell}}_i &= \frac{\partial [g^{\ell}(z,x^{\ell})]_i}{\partial z} \nonumber \\
\delta^{f^a}_i &= (\nabla f|_{S})_{p^s + i}  \hspace{20pt}
&\delta^{f^{\ell}}_i &= (\nabla f|_{S})_{p^s + p^a + i} \nonumber 
\end{align}
recalling assumptions~\ref{asm:fderiv},\ref{asm:gderiv} to assert that these partial derivatives exist. Assumptions~\ref{asm:fderiv} and~\ref{asm:gderiv} further imply $\delta^{f^{\ell}}_i>0$, $\delta^{f^a}_i > 0$ and $\delta^{g^a}_i < 0$, $\delta^{g^{\ell}}_i < 0$ respectively, so 
\begin{align}
h'(z) &= \mathbb{E}_{X^{\ell}}\left[  \sum_{i}^{p^a} \delta^{g^a}_i \delta^{f^a}_i  + \sum_{i}^{p^{\ell}} \delta^{g^{\ell}}_i \delta^{f^{\ell}}_i  \right] < 0 \label{eq:hderiv} 
\end{align}
and thus the recursion $\rho_{e+1} = h(\rho_{e})$ has exactly one fixed point. Call this $z_0$, so $z_0=h(z_0)$. We now note
\begin{theorem}
\label{thm:naive_updating_behaviour}

If $h'(z_0) \leq -1$ then the recursion does not converge unless $\rho_0=z_0$, and will tend toward a stable oscillation between two values. If for some (possibly unbounded) interval $R$ we have $\rho_e \in R$ for some $e$ and for all $z \in R$, $h(z) \in R$ and 
\begin{align}
\sum_{i}^{p^a} \left( \delta_i^{g^a} \right)^2 &\leq k_1,  & 
\sum_{i}^{p^{\ell}} \mathbb{E}_{X^{\ell}} \left[ \left( \delta_i^{g^{\ell}} \right)^2\right] &\leq k_2 \label{eq:gcond} \\
\sum_{i}^{p^a} \mathbb{E}_{X^{\ell}}\left[  |\delta_i^{f^a} |\right]^2 &\leq k_3, &
\sum_{i}^{p^{\ell}} \mathbb{E}_{X^{\ell}}\left[  \left(\delta_i^{f^{\ell}}\right)^2 \right] &\leq k_4 \label{eq:fcond}
\end{align}
where $\sqrt{k_1 k_3} + \sqrt{k_2 k_4} < 1$, then 
\begin{equation}
|\rho_e(x^s,x^a)-\rho_{e+1}(x^s,x^a)| \to 0 \nonumber
\end{equation}
as $n \to \infty$. 

\end{theorem}

This is proved in Supplementary Appendix~\ref{supp_sec:thm1proof}. 
Alternative conditions for convergence (`performative stability') are proved in~\cite{perdomo20}.

Condition~\eqref{eq:gcond} states that, on average, interventions make only small change to $x^a$ and $x^{\ell}$ in response to small changes in $\rho$. Condition~\eqref{eq:fcond} states that, on average, the actual risk changes little with small changes in covariates. These conditions are sufficient but not necessary. Since $h'(z)<0$, successive estimates of $\rho_e$ will oscillate around their limit. In general, a requirement for general convergence of $\rho_e$ restricts the type of interventions which can be in place. A simple scenario in which $\rho_e$ cannot converge is provided in Supplementary Section~\ref{supp_sec:oscillation}, and we illustrate an example showing convergence and divergence of $\rho_e$ in Figure~\ref{fig:main_example}. 
We produced a simple web app illustrating this problem at \url{https://ajl-apps.shinyapps.io/universal_replacement/}

\begin{figure}[p]
\centering
\includegraphics[width=0.7\textwidth]{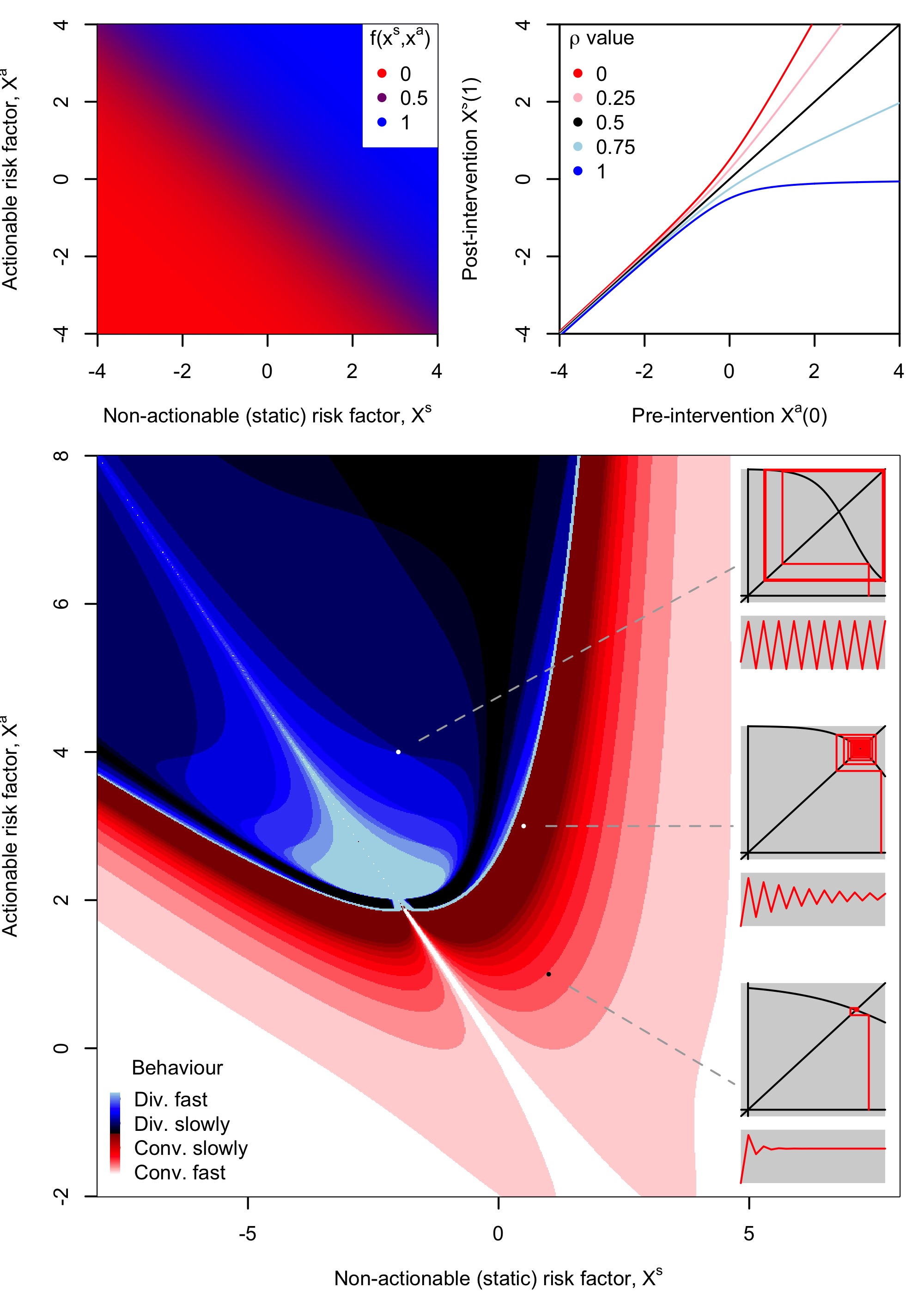}
\caption{Example showing convergence and divergence of $\rho_e$ across epochs. We disregard $x^{\ell}$, $g^{\ell}$ in this example. We choose $f(x^s,x^a)=\textrm{logit}(x^s,x^a)$ (top left). We choose $g^a$ with the rationale that we intervene by lowering $X^a(0)$ when $\rho_e> 1/2$, but allow $X^a(0)$ to increase when $\rho_e< 1/2$ (that is, resources for intervention are redistributed rather than introduced), and assume that we can intervene more effectively when $X^a(0)$ is high (
strictly, $g^a(\rho,x^a) = \frac{1}{2}\left((3-2\rho)x^a + (1-2\rho)\sqrt{1+(x^a)^2}\right)$, 
top right panel). Bottom panel shows whether $\rho_e(x^s,x^a)$ converges or diverges, and how long it takes (num. epochs until $\Delta_e \triangleq |\rho_e-\rho_{e-1}|<0.01$ or $(|\Delta_e|>0.05 \cup 
|\Delta_e-\Delta_{e-1}|<0.01)$; $|e|\leq 10$). Insets show cobweb plots for relevant recursions, and plots of $\rho_e$.}
\label{fig:main_example}
\end{figure}

We may hope that naive updating, when it converges, may solve the optimisation problem in Section~\ref{sec:aim}. It does not, and we give a specific counterexample in Supplementary Section~\ref{supp_sec:nonoptimal}. Finally, we note that the dynamics above also model a related setting, where samples are tracked across epochs and interventions are permanent (Supplementary Section~\ref{supp_sec:alternative}). In summary, naive updating can readily lead to wide oscillation of successive risk estimates, and even if $\rho_e$ does converge, the limit does not generally correspond to an optimal outcome in terms of minimising incidence of $Y$.
\section{Successive adjuvancy}
\label{sec:successive_adjuvancy}

\emph{Note: This section and associated content (theorem~\ref{thm:successive_adjuvancy}, figures~\ref{fig:chaos},~\ref{fig:causality_sa}, and associated discussions) were not included in the originally submitted version of this paper due to length. This material does not appear in the published version of this manuscript, and the reader should be aware that these sections did not undergo peer review.}
\vspace{10pt}


We propose a second strategy for updating risk scores in which interventions are `built' across successive epochs, effectively using new risk scores as adjuvants to risk scores from previous epochs, rather than replacements. 

We retain assumptions~\ref{asm:first_main_assumption} through~\ref{asm:gderiv} except assumption~\ref{asm:equally_distributed}: we assume that $f_e$ and $\mu_e$ remain fixed across epochs, but $g^a_e$ and $g^{\ell}_e$ do not. Although we no longer consider $g^a_e$ and $g^{\ell}_e$ fixed across epochs, we consider fixed functions $g^a$ and $g^{\ell}$ which will be used as `building blocks' for $g^a_e$ and $g^{\ell}_e$. In epoch $e$, we observe initial values $x_e(0)=\left(x^a_e(0),x^s_e(0),x^{\ell}_e(0)\right)=(x^a_e,x^s_e,x^{\ell}_e)$ at $t=0$, and compute $\rho_0(x^a_e,x^s_e)$, $\rho_1(x^a_e,x^s_e)$, $\dots$, $\rho_{e-1}(x^a_e,x^s_e)$. 

We build $g^a_e$, $g^{\ell}_e$ as follows. We begin by intervening on $x^s_e(0),x^{\ell}_e(0)$ according to $\rho_0$ and the building block functions $g^a$, $g^{\ell}$ to get $g^a(\rho_0,x^a_e)$, $g^l(\rho_0,x^{\ell}_e)$. We then intervene on these new values according to $\rho_1$, to get $g^a\left(\rho_1,g^a(\rho_0,x^a_e)\right)$, $g^{\ell}\left(\rho_1,g^{\ell}(\rho_0,x^{\ell}_e)\right)$. We then intervene on these values according to $\rho_2$, and so on. The intervention functions at epoch $e$ are thus defined as
\begin{align}
g^a_e(\rho,x^a) &= g^a\left(\rho_{e-1},g^a\left(\rho_{e-2},\dots, g^a(\rho_0,x^a)\dots\right)\right)\nonumber \\
g^{\ell}_e(\rho,x^{\ell}) &= g^{\ell}\left(\rho_{e-1},g^{\ell}\left(\rho_{e-2},\dots, g^{\ell}(\rho_0,x^a)\dots\right)\right)\nonumber \\
\end{align}
taking $x^s$ at some fixed value, and $\rho_0$, $\rho_1$, $\dots$, $\rho_{e-1}$ as fixed functions. We also presume again that $\rho_e$ is an oracle; that is, that the approximation in equation~\ref{eq:rho_oracle} is perfect. This enables construction of a recursive definition:
\begin{align}
g^a_0(\cdot,x^a) &= x^a \nonumber \\
g^{\ell}_0(\cdot,x^{\ell}) &= x^{\ell} \nonumber \\
\rho_0 = \rho_0(x^s,x^a) &= \mathbb{E}_{X^{\ell}} [f\left(x^s,x^a,X^{\ell}\right)] \nonumber \\
& \nonumber \\
g^a_1(\rho_0,x^a) &= g^a(\rho_0,x^a) \nonumber \\
g^{\ell}_1(\rho_0,x^{\ell}) &= g^{\ell}(\rho_0,x^{\ell}) \nonumber \\
\rho_1 = \rho_0(x^s,x^a) &= \mathbb{E}_{X^{\ell}} [f\left(x^s,g^a_1(\rho_0,x^a),g^{\ell}_1(\rho_0,X^{\ell})\right)] \nonumber \\
& \nonumber \\
g^a_{e+1}(\rho_e,x^a) &= g^a(\rho_{e},g^a_e\left(\rho_{e-1},x^a)\right) \nonumber \\
g^{\ell}_{e+1}(\rho_e,x^{\ell}) &= g^{\ell}(\rho_{e},g^{\ell}_e\left(\rho_{e-1},x^{\ell})\right) \nonumber \\
\rho_{e+1} = \rho_{e+1}(x^s,x^a) &= \mathbb{E}_{X^{\ell}} [f\left(x^s,g^a_{e+1}(\rho_e,x^a),g^{\ell}_{e+1}(\rho_e,X^{\ell})\right)] \label{eq:sarecursion}
\end{align}

\subsection{Dynamics of successive adjuvancy}

The dynamics of this system are more complex than that of naive updating. However, under much simplified circumstances: a univariate $x^a$, and disregarding $x^l$, we show the following:


\begin{theorem}
\label{thm:successive_adjuvancy}

Assume the following:
\begin{enumerate}
\item $g^{\ell}(\cdot,x^{\ell})=g^{\ell}_e(\cdot,x^{\ell})=x^{\ell}$, and $X^{\ell}_e \sim \delta_0$ (so all terms involving $\ell$ can be omitted from recursion~\ref{eq:sarecursion}) 
\item $x^a$ is univariate ($p^a=1$) \label{asm:univ}
\item $\frac{\partial}{\partial x^a}f(x^s,x^a)>0$ \label{asm:fpartial} \label{asm:partial}
\item For some unique $\rho_{eq}$ we have $\forall x \hspace{5pt} g^a(\rho_{eq},x) =x$, and $\forall (x,\rho \neq \rho_{eq}) \hspace{5pt} g^a(\rho,x)\neq x$ 
\end{enumerate}
For brevity we define $f(x)=f(x^s,x)$ and denote by $S_2$ the substitution $(\rho,x^a) \to (r,f^{-1}(r)$.  Now if, for some interval $I$ containing $\rho_{eq}$, we have $\rho_e \in I$ for some $e < \infty$, and for all $r \in I$, we have
\begin{equation}
\left|\frac{\partial}{\partial r} f(g^a(r,f^{-1}(r))) \right|  = \left|f'\left(g^a(r,f^{-1}(r)\right)\left(\frac{\frac{\partial g^a}{\partial x^a}|_{S_2}}{f'(f^{-1}(r))} + \frac{\partial g^a}{\partial r}|_{S_2} \right) \right| < 1 \label{eq:h2deriv}
\end{equation}
then
\begin{align}
\rho_e(x^s,x^a) &\to \rho_{eq} \nonumber \\
P\left(Y_e|(X^s_e(0),X^a_e(0))=(x^s,x^a)\right) &\to \rho_{eq} \nonumber \\
g^a_e(\rho_{e-1},x^a) &\to \{x: f(x^s,x)=\rho_{eq}\}= f^{-1}(\rho_{eq})
\end{align}
as $e \to \infty$. 

\end{theorem}

This is proved in Supplementary Section~\ref{supp_sec:thm2proof}.
Although limited to simplified circumstances, this results of this theorem warrant some interpretation. We may consider $\rho_{eq}$ to be an `equivocal risk': that is, a risk at which the value of $x^a$ remains the same. The theorem roughly states that, for sufficiently slowly-changing $f$ and $g^a$, interventions will build towards a point in which interventions bring everyone to almost the same (equivocal) risk level. 

For certain reasonable values of $f$ and $g^a$, including those used for figure~\ref{fig:main_example}, the derivative of $h_2$ can change sign, leading to chaotic behaviour of $\rho_e$ (figure~\ref{fig:chaos}).

\begin{figure}[h]
\centering
\includegraphics[width=0.45\textwidth]{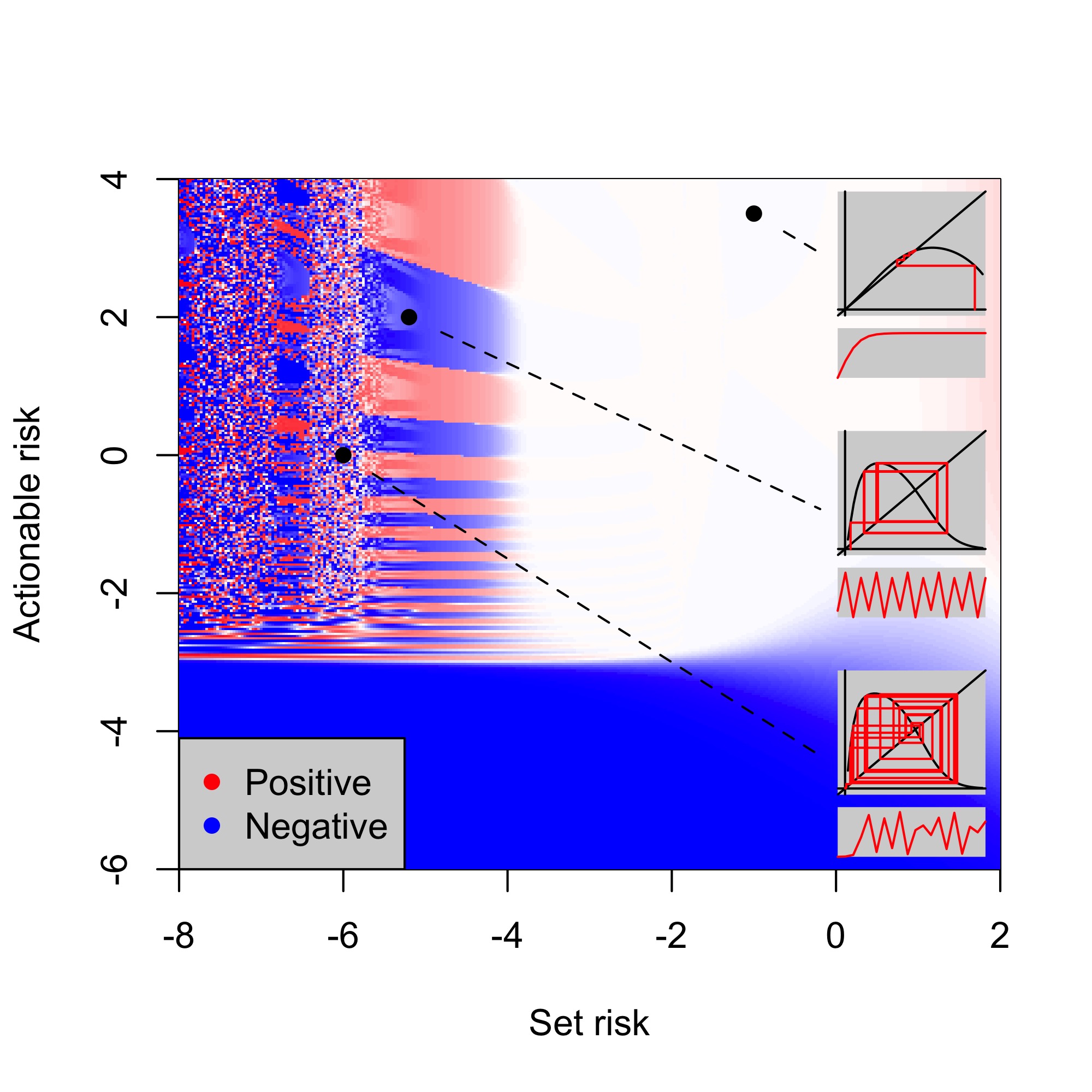}
\includegraphics[width=0.45\textwidth]{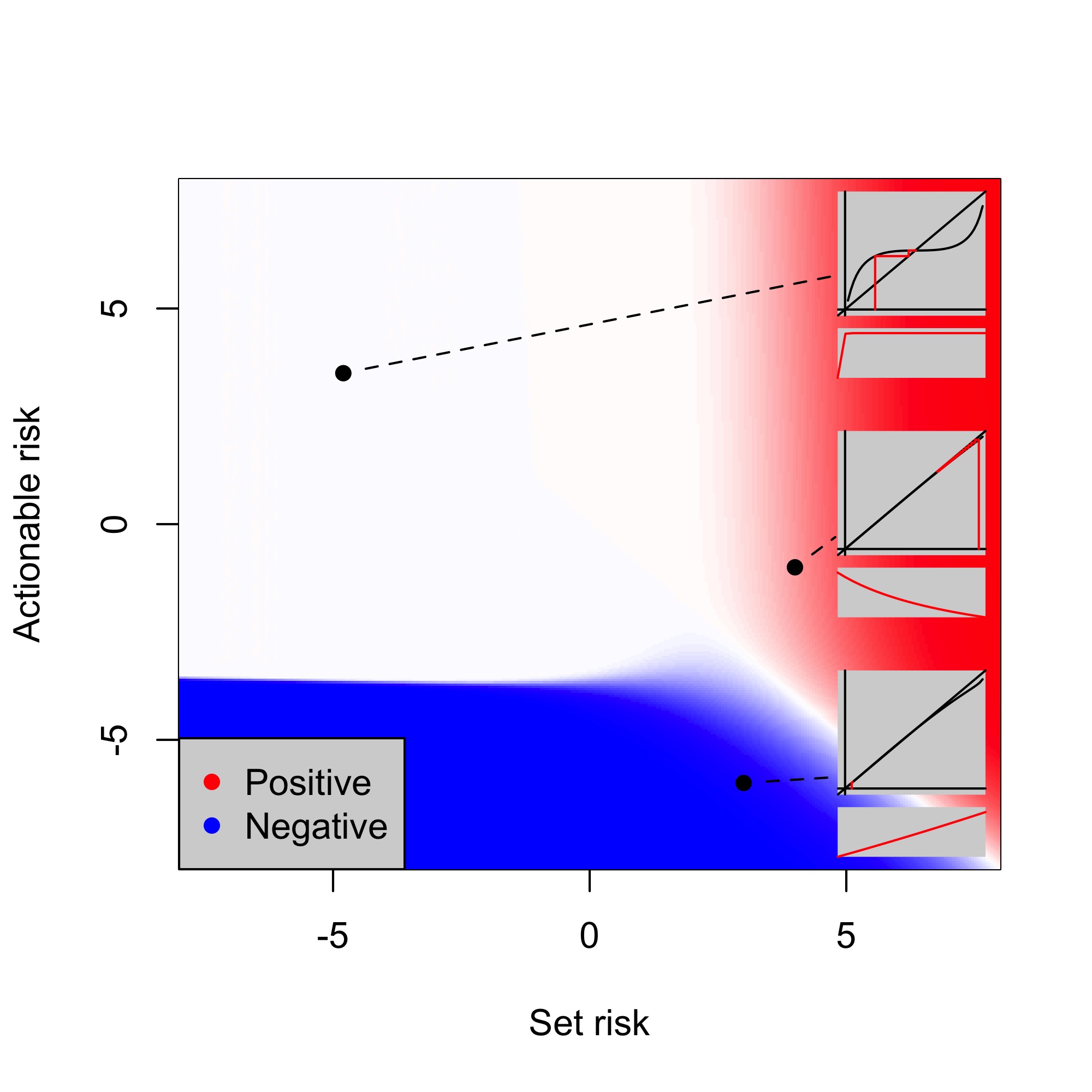}
\caption{Dynamics of successive adjuvancy. In both panels, $f(x^s,x^a)=\textrm{logit}(x^s+x^a)$, we have $\rho_{eq}=1/2$ and the colour indicates the difference $\rho_e-\rho_{eq}$ for $e=20$. The left panel shows dynamics in which $g^a$ has the same form as for figure~\ref{fig:main_example}, and can be seen to lead to chaotic behaviour of $\rho_e$ for some values of $(x^s,x^a)$. The right-hand panel uses $g^a(\rho,x^a) = x^a - 4(1-\rho)\textrm{logit}(x^a)$, and $\rho_e$ can be seen to converge to $\rho_{eq}$ everywhere, albeit at different rates.}
\label{fig:chaos}
\end{figure}

An advantage of successive adjuvancy over naive replacement is that risk scores from previous epochs $\rho_0$, $\rho_1$, $\dots$, $\rho_{e-1}$ have an immediate interpretation as  unbiased estimates of risk  estimates through the process of intervention. If we consider the interventions $g^a_e$, $g^{\ell}_e$ as a series of interventions of type $g^a$, $g^{\ell}$ applied in succession, then $\rho_0$ is the true risk ($P(Y)$) before applying $g^a$, $g^{\ell}$ at all, $\rho_1$ is the true risk after applying $g^a$, $g^{\ell}$ once in response to $\rho_0$, $\rho_2$ is the true risk after applying $g^a$, $g^{\ell}$ firstly in response to $\rho_0$ and subsequently in response to $\rho_1$, and so on. Specifically, $\rho_{e-1}$ is the risk of $Y$ immediately before applying $g^a$, $g^{\ell}$ for the final ($e$th) time. When used for this final time, $g^a$ and $g^{\ell}$ are applied in response to $\rho_{e-1}$ itself. Figure~\ref{fig:causality_sa} illustrates this idea for epochs 0 and 1 using the format of figure~\ref{fig:causality}. Seen in this way, repeatedly adjusting covariates on the basis of new risk estimates resembles a `boosting' strategy, in which each new $\rho_e$ captures the residual risk from $\rho_0$ through $\rho_{e-1}$, which seems a logical approach in a real-world situation.

\begin{figure}[h]
\centering
\includegraphics[width=0.5\textwidth]{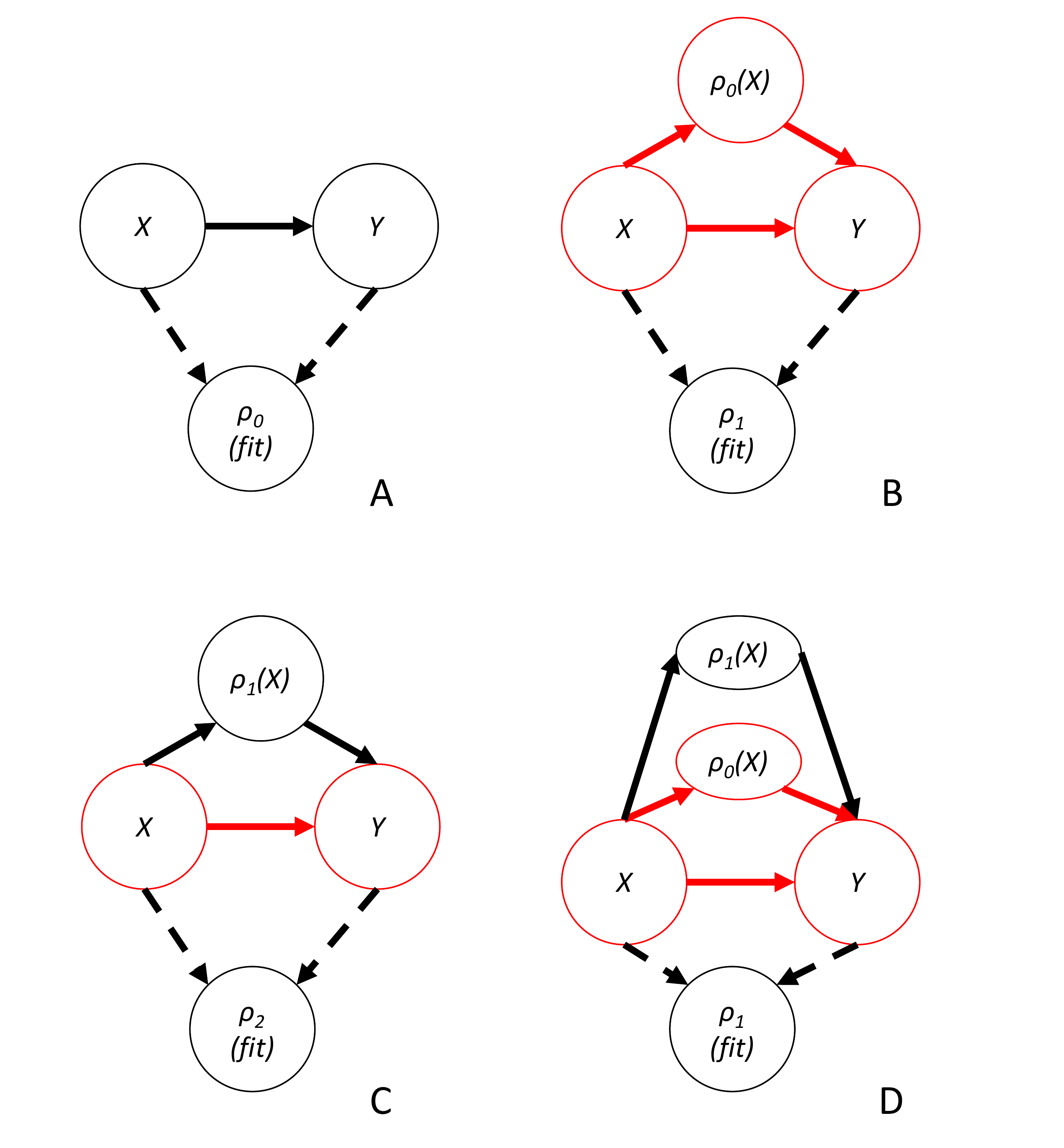}
\caption{Underlying causality structure of successive adjuvancy. Panel A shows structure in a score-naive setting where $\rho_0$ is fitted. Panel B shows the causal structure when $\rho_1$ is fitted; as for figure~\ref{fig:causality}, $\rho_1$ models a setting (coloured red) in which $\rho_0$ forms an additional causal pathway from $X$ to $Y$. Undernaive updating (panel C), $\rho_1$ replaces $\rho_0$, and is used to model a scenario distinct from that to which it was fitted. However, in successive adjuvancy (panel D), $\rho_1$ is an adjuvant to $\rho_0$, and thus still operates on the same system to which it was fitted. In panel D, $\rho_1$ is used to guide interventions after having intervened on $\rho_0$.}
\label{fig:causality_sa}
\end{figure}

An implementation of successive adjuvancy is included in our web app at \url{https://ajl-apps.shinyapps.io/universal_replacement/}.


\section{Strategies to avoid this problem}
\label{sec:solution}

Naive updating is an appropriate method for updating risk scores if no interventions are being made (that is, $g^a(\rho,x^a)=x^a$ and $g^{\ell}(\rho,x^{\ell})=x^{\ell}$), as may be the case if a risk score is used for prognosis only, rather than to guide actions\footnote{EUROscore2~\citep{nashef12} (a risk predictor for cardiac surgery) can be used in this way, by giving patients prognostic estimates but without being used to recommend for or against surgery}. It may also be appropriate if we do not aim to solve the constrained optimisation problem in Section~\ref{sec:aim}, and are only concerned with accuracy of the model: in that case, under at least the conditions of Theorem~\ref{thm:naive_updating_behaviour}, naive updating will lead to estimates $\rho_e(x^s,x^a)$ converging as $e \to \infty$ to a setting in which $\rho_e$ accurately estimates its own effect: conceptually, $\rho_e(x^s,x^a)$ estimates the probability of $Y$ \emph{after} interventions have been made on the basis of $\rho_e(x^s,x^a)$ itself~\citep{perdomo20}. Naive updating is otherwise generally not advisable, although a range of alternative modelling strategies do not lead to the same problems.

We demonstrate three general strategies for avoiding the naive updating problem below. We describe how each of these accomplishes this and compare their advantages in Supplementary section~\ref{supp_sec:solution_comparison}. We describe how an implementation of each strategy may look in the context of a toy example in supplementary section~\ref{supp_sec:solution_illustration}.

Successive adjvuancy may be an appropriate method for updating risk scores if eventual convergence can be proven and a progression of all samples towards the same risk level is be a desirable outcome. Such an outcome clearly does not generally solve the constrained optimisation problem in~\ref{sec:aim} as the cost may be arbitrarily large. Although $g^a_e$ and $g^{\ell}_e$ are variable, they are entirely built of successive applications of $g^{a}$ and $g^{\ell}$, which may not be practical. 

\subsection{More complex modelling and more data}
\label{sec:solution_modelling}

An obvious way to avoid the problem is to model the setting completely, including the effect of any interventions. Methods of this type would include explicit causal modelling, as used in related problems~\citep{sperrin18}, or counterfactual inference, which has been suggested as a direct approach to the problem~\citep{sperrin19}. These approaches would require knowledge or accurate inference of $g^{\ell}$ and $g^a$, or observation of covariates at several points in each epoch~\citep{sperrin18}. 

A second approach is to consider data from previous epochs alongside the current data when fitting $\rho_e$. Such data can be used as a prior on the fitted model~\citep{alaa18} and could be used to infer model elements: $\mu_e$, $g^{\ell}$, $g^a$, and $f$. If accurate data were available, oscillatory effects could even be detected and avoided. A difficulty with this approach in a realistic setting is in distinguishing whether inaccuracies in older models are due to drift in the underlying system \citep{Quionero-Candela2009-al} (in our case, $f$ and $\mu_e$) or due to the effects of intervention. Indeed, the problems with naive updating can be seen as treating model inaccuracies as though they are due to the first effect, when they are in fact due to the second. Definitive assertion of the cause of inaccuracies will, again, generally require more frequent observation of covariates.

\subsection{Hold out set}
\label{sec:solution_holdout}

A straightforward and potentially practical means to avoid the problems associated with naive updating is to retain a set of samples in each epoch for which $\rho_e$ is not calculated, and hence cannot guide intervention. For such samples, $X_e(0)=X_e(1)$, so a regression of $Y$ on $X_e(0)$ restricted to these `held out' samples can be used as an unbiased estimate for $f_e$. If the hold out set is randomly selected, this would emulate a \emph{clinical trial} which enables us to assess the effect of predictive scores (and their associated interventions) across epochs. 

A problem with this approach is that any benefit of the risk score-guided intervention is lost for individuals in the hold-out set. Careful consideration of the ethical consequences of this strategy is therefore required.

\subsection{Control interventions}
\label{sec:solution_control}

A radically different option is the direct specification of the interventions $g^{\ell}_e$ and $g^a_e$ in each epoch, considering $\rho_e$, $\mu_e$ constant, and $f_e$ to change only slightly with $e$. This enables directly addressing the constrained optimisation problem in Section~\ref{sec:aim}. 

If $X^{\ell}$ can be disregarded, and we may regard $f_{e-1}$ as an unbiased estimate of $f_e$\footnote{This assumption underlies the fundamental point of a risk score}, then we may take a simple inductive approach:
\begin{enumerate}
\item At the end of epoch 0, infer $f_0$ and $\mu_0$. Given some fixed functions $\rho$, $c^a$, find a function $g^a_1$ which solves the constrained optimisation problem in section~\ref{sec:aim} assuming $f_1=f_0$, $\rho_1=\rho_0$. Implement this intervention.
\item At the end of epoch $e>0$, regress $Y_e$ on 
\begin{equation}
X_e(1)=\left(X^s_e(0),g^a_e\Big(\rho(X^s_e(0),X^a_e(0)),X^a_e(0)\Big)\right) \nonumber 
\end{equation}
to attain an unbiased estimate of $f_e$. Now solve the constrained optimisation problem to optimise $g^a_{e+1}$, assuming $f_{e+1}=f_e$ and $\rho_{e+1}=\rho_{e}$
\end{enumerate}
Thus in each epoch an unbiased update of $f_e$ can be made, and the constrained optimisation problem can be directly solved.
If $X^{\ell}$ is present, the problem is more complex. We suggest this general case as an open problem (see Supplementary Section~\ref{supp_sec:open}).

A problem with this approach in a medical setting is that specification of $g^a_e$ may cause the procedure to be subject to medical device regulation~\citep{mhra19}. Implications of these regulatory processes map to our potential solutions; for example, countries in the EU~\citep{eu17} have only developed regulatory processes to the point of accommodating static risk scores, and by extension currently treat updated scores as new tools. In these cases a separate evaluation exercise, such as testing on a hold-out, is necessary to demonstrate efficacy prior to dissemination, which would also remedy the problems of naive updating (although costs of repeated formal evaluations of effectiveness, and the ethics of a hold-out, may be a concern). However, the US FDA have proposed an alternative `total-life-cycle' approach~\citep{fda19} which allows for model updating (contingent on defining a performance monitoring mechanism), which, given the problems of naive updating, is potentially seriously flawed.

\section{Formulation as control-theoretic/ reinforcement learning problem}
\label{sec:control}

Control theory \citep{Bertsekas1995-hr} and its modern incarnation, reinforcement learning \citep{Sutton2018-ng}, study temporal problems where multiple actions are available at each time step. The aim of the field is to come up with an optimal policy either from the start or, in the partially observable case, a mechanism that quickly converges to the optimal policy. In the latter the regret is considered to be how much utility is lost compared to using the optimal policy from the start. The methods underlying this, like dynamic programming, are used in a variety of fields such as; playing go~\citep{Silver2018-alphago},  
in dynamic treatment strategy~\citep{alaa18} and mechanical and electrical engineering. 
Here we use the formulation of a Partially Observable Markov Decision Processes (POMDP) \citep{Yuksel2017-ni}, and adopt the notation from \citep{Wang_undated-wi} whereby we consider the POMDP as a 7-Tuple $\left(\mathcal{S},\mathcal{A},\mathcal{T},\mathcal{R},\Omega,\mathcal{Z},\gamma\right)$:
\begin{itemize}
\item $\mathcal{S},\mathcal{A}$ and $\Omega$ are spaces of states, actions and observations.
\item $\mathcal{T}$ is the transition kernel that describes the evolution given state and action, e.g. $s_{e+1}\sim\mathcal{T}(\cdot\mid s_{e},a_{e})$
(i.e. a set of conditional transition probabilities between states and actions).
\item $\mathcal{Z}$ is a kernel for the observation given the state, e.g. $o_{e+1}\sim\mathcal{Z}(\cdot\mid s_{e},a_{e})$\footnote{Note that here future observations depend on current states and actions and not on future states and actions}.
\item $r_{e}$ represents our reward for being in state $s$ and taking action $a$ at time (or equivalently epoch) $e$, and is sampled from $\mathcal{R}$ - i.e. $r_e\sim\mathcal{R}(s_{e},a_{e})$ 
\item $\gamma$ is a discount factor that down-weighs future rewards if $0<\gamma<1$.
\end{itemize}
A solution candidate is a policy 
$$a_{e}\sim\pi\left(\left\{ o_{s},r_{s},a_{s}\right\}_{s=1}^{e-1}\right)$$
which aims to maximise 
\[\mathbb{E} \sum_{e=1}^{M}\gamma^{e-1} r(s_e,a_e)
\]
where $M$ represents the maximum number of time/epoch steps. Other reward/utility parametrisations are possible e.g. to include a final pay off or infinite time horizon pay off. Several options for reward function construction are detailed in~\citep{liu14,yu19,wirth17}.
The beauty of this framework is the flexibility: aspects such as optimisation under uncertainty can be included by including parameters of reward, transition and observation processes into the (unobserved) state variable. 

We cast the above in this framework: 
\begin{align}
s_{e}&=\left(X_{e}(0),X_{e}(1),Y_{e}\right) \nonumber  \\
a_{e}&=\rho_e \nonumber \\
o_{e}&=\left((X_{e}^{s}(0),X_{e}^{a}(0)),Y_{e}\right) \nonumber \\
r_{e}&=\mathbb{P}\left(\bar{Y}_{e+1}\mid s_e,a_{e}\right) \nonumber
\end{align}
with $\bar{Y}$ corresponding to the rate of events in total population.

The transition kernel from $s_e$ to $s_{e+1}$ consists of; sampling $X_{e+1}(0)$ (note that this sampling is independent of $s_e$), intervening using this sample with $\rho_e$ to form $X_{e+1}(1)$, and then using these values to sample $Y_{e+1}$ from the resulting conditional distribution. Finally we note that given Assumption~\ref{asm:fourth_main_assumption} our policy $a_e \sim \pi(o_e,r_e,a_e)$ as previous epochs are ignored. Indeed, this assumption also implies that $s_{e+1},o_{e+1}$ and $r_e$ only depend on the previous state through $a_e = \rho_e$. 
In the control view point it is also easy to formulate the longitudinal problem (this corresponds to setting $X_{e+1}(0)=X_{e}(1)$).

The description above allows to use methods of the field such as Q-learning, (approximate dynamic programming), PDE-based approaches such as the Hamilton Jacobi Bellman equation and many more. 
These methods create a policy which maps the historical observations to an action (for the problem at hand a risk score function).
Most of the rigorous methods require a low dimensional state space \citep{Powell2007-xe}.

\section{Discussion}
\label{sec:discussion}

In this work, we elaborate on the issue raised by Lenert and Sperrin~\citep{lenert19,sperrin19} and propose a framework for quantitatively modelling its effects, with a particular focus on a model which is updated repeatedly. We demonstrate some consequences of ignoring this problem, and note that they occur even in highly idealised circumstances. Although the problem can generally be avoided by more complex and complete modelling, we consider that this is often impractical: a full consideration of the setting in which a model will eventually be used is not generally considered until the model is to be implemented~\citep{lipton18}.

The formulation of the constrained optimisation problem in section~\ref{sec:aim} makes it clear that for fixed $g^{\ell}$, $g^a$, the best possible $\rho_e$ is not necessarily the oracle estimator in equation~\ref{eq:rho_oracle}. However, many machine learning models tend to focus on accurate prediction of outcomes~\citep{nashef12}, rather than directly solving problems of the type in section~\ref{sec:aim}; hence, the naive updating setting considers a $\rho_e$ which does exactly this. In the naive updating setting, we are assuming an analyst who ignores this effect.

The model presented here is not a full description of modern predictive scoring systems; however, it is extensible in various ways (some detailed in Supplementary Section~\ref{supp_sec:open}). In particular, $g^{\ell}$ and $g^a$ could be random-valued rather than deterministic. We also note that we assume a covariate value after intervention confers the same contribution to risk of $Y$ as it does when it takes the same value `naturally', which may not be realistic.

We assume we are `starting over' with new samples at the beginning of each epoch, and for naive updating, we assume that covariate values are identically distributed. The basis for this assumption is that we generally expect interventions to be zero-sum: that is, the risk score guides a redistribution of intervention rather than introduction of interventions, so the total effect on the sample population remains roughly the same in each epoch. In this assumption, we differ from that in the analysis by Lenert [2019]. We can alternatively interpret this assumption as taking all interventions as being short-term and having `worn off' by the start of the next epoch. The problem raised here also exists for the more general setting when interventions have long term effects and we consider longitudinal effects.

An important consideration in model updating is `stability' of successive predictions: in our setting, whether successive values of $\rho_e$ converge. Colloquially, we can take 'stability' to mean that if the underlying system being modelled does not change, then updating a model will leave it unchanged; the model predicts its own effect.  General conditions for stability are considered in~\cite{perdomo20}
, who differentiate between stability in which $\rho$ optimises a loss given its own effect, and `performative optimality', in which $\rho$ globally optimises a loss. 
Although we highlight that stability does not generally guarantee that the model is getting the best outcome (according to the constrained optimisation problem in section~\ref{sec:aim}), we note that stability has real-world advantages: in particular, trust in a model will generally be better if it appears to be stable.

In the setting where models change at each epoch, if $m_{\tilde{f}_{e}}$ is known at the current epoch $e$, we note a fair comparison of models is one which compares models built using the training data available at the current epoch\footnote{This is not to say that the performance of models will not deteriorate over epochs, just that the issue may not lie with the model structure.}. If $m_{\tilde{f}_{e}}$ is not known, then a holdout set for test data must be used so a fair comparison can be made using an estimate of $m_{\tilde{f}_0}$ (assuming $\tilde{f}_0 \approx f$). This is because at epoch $e$ we only have access to $(X_e(0),Y_e)$ and not $X_e(1)$, and so we are not able to properly gain insight to the behaviour of $\tilde{f}_e$ needed to provide an estimate of $m_{\tilde{f}_e}$. An attempt to estimate $m_{\tilde{f}_e}$ using $(X_e(0),Y_e)$ implicitly assumes that $Y_e$ directly depends on $X_e(0)$, and as a result $\rho_e$ would appear much closer to $\tilde{f}_e$ than is the case. Put simply, by implementing naive model updating not only may performance severely worsen (even if better models were used), but in not providing a holdout test set stakeholders may not even be able to recognise that performance is worsening as the number of epochs increase.

In essence, we provide a causal framework within which to understand a crucial issue in regulation of machine learning and AI-based tools in health and further afield, demonstrating that approaches which incorporate naive updating are unlikely to be fit for purpose. Moreover, even where solutions are available to address the bias introduced by updating on `real-world' data in which outcomes represent (at least in part) the effects of an algorithm, these restrict the potential of `online' and frequently updated solutions. We hope that our work will foster discussion of this interesting problem, which is becoming increasingly  pertinent as machine-learning based predictive scores become widely used to guide decision making, and policymakers act to address how to regulate these tools to ensure safety and effectiveness.

\subsection*{Code availability}

Code to reproduce relevant plots and examples is available at~\url{github.com/jamesliley/model_updating}.

\subsubsection*{Acknowledgements}

We thank the Alan Turing Institute, MRC Human Genetics Unit at the University of Edinburgh, Durham University, University of Warwick, Wellcome Trust, Health Data Research UK, and Kings College Hospital, London for their support of the authors. 
%
%
This problem was first identified in our circumstance by LJMA. We thank Dr Ioanna Manolopoulou for helping to draw our attention to the imminence of this problem.

JL, CAV and LJMA were partially supported by Wave 1 of The UKRI Strategic Priorities Fund under the EPSRC Grant EP/T001569/1, particularly the ``Health'' theme within that grant and The Alan Turing Institute;
JL, BAM, CAV, LJMA and SJV were partially supported by Health Data Research UK, an initiative funded by UK Research and Innovation, Department of Health and Social Care (England), the devolved administrations, and leading medical research charities;
SRE is funded by the EPSRC doctoral training partnership (DTP) at Durham University, grant reference EP/R513039/1;
LJMA was partially supported by a Health Programme Fellowship at The Alan Turing Institute;
CAV was supported by a Chancellor's Fellowship provided by the University of Edinburgh. 

\clearpage

\bibliographystyle{plain}
\bibliography{references}

\clearpage

\makeatletter\@input{xx.tex}\makeatother
\makeatletter\@input{yy.tex}\makeatother
\end{document}


\maketitle

\section{Example of functions and variables in a realistic setting}
\label{supp_sec:realistic_exposition}

We consider the model proposed by~\cite{rahimian18} for prediction of emergency admission to a hospital in a given time period on the basis of electronic health records (EHRs). Such a model is not in common use in the location considered (England), so the data in the original paper is not affected by the problems we describe in the main manuscript.

For clarity\footnote{Analogous times and variables can be described for other prediction periods and updating patterns}, we presume a prediction window of ten months (February-November), and that predictions are made and distributed to primary health practitioners in January, with a new model being trained on the basis of each year's data in December, to be implemented the following January. In this setting, distribution of the score may open a second causal pathway between covariates and outcome as shown in figure~\ref{fig:causality}, and is thus susceptible to the problems of naive updating.

In this setting, variables and functions may be interpreted as follows:
\begin{enumerate}
\item $Y$ the event `an emergency admission in the following year'
\item $X_e(0)$ the values of all variables which affect $E(Y)$ at the time when the predictive score is computed (the start of each year)
\item An `epoch': the time in which a given model is in use; eg, each year.
\item `Time': $t=0$ when the predictive score is computed (the start of January); $t=1$ represents the time after which any interventions are made (the start of Feburary). 
\item $X^s_e$ covariates affecting $\mathbb{E}(Y)$ which are included in the predictive score but which cannot be directly modified in the time frame: age, time since most recent emergency admission
\item $X^a_e$ covariates affecting $\mathbb{E}(Y)$ included in the predictive score which can be modified in the time frame: current medications.
\item $X^{\ell}_e$ covariates affecting $\mathbb{E}(Y)$ which are not included in the predictive score, and possibly can be modified in the time frame: blood pressures, cardiac function
\item $f_e$ the underlying causal process for $Y$ given patient status; that is, the probability of admission in the subsequent year, given covariates.
\item $g^a_e$ Hypothetical prescribed interventions made on $X^a$ in response to a predictive score; for instance, reduce drug dosages. We roughtly assume that this intervention is symmetric; for a patient at low emergency risk, a higher drug dose is acceptable.
\item $g^{\ell}_e$ Hypothetical prescribed interventions made on $X^{\ell}$ in response to a predictive score; for instance, treat low or high blood pressure.
\end{enumerate}
It is clear that if such a risk score were used universally, and data was collected from the period in which a model was in place was then, then the data would be affected by the effect of the predictive score itself. 

The model does not fully describe this setting. The trichotomisation into $X^{\ell}$, $X^a$, and $X^s$ is not perfect; intervention on $X^L$ could also affect some variables in $X^a$ and vice versa. Interventions are likely to be random-valued to some extent.



\section{Alternative system described by naive updating}
\label{supp_sec:alternative}

We note that the definition of $h$ (equation~\eqref{eq:hdef} in main text), and hence the following comments on recursion dynamics, can be used to describe a related setting in which we track the same samples over epochs, and the effect of interventions $g^a$, $g^{\ell}$ remain in place. Formally, we retain definitions of $X^s,X^a,X^{\ell},e,t,f_e,g^a_e,g^{\ell}_e,\rho_e$ and all assumptions except~\ref{asm:ident_dist},\ref{asm:equally_distributed} from the main text. In place, we assume that $f_e$, $g^a_e$, $g^{\ell}_e$ are fixed across epochs, but instead of resampling $X_e(0)$ from $\mu_e$, we have
\begin{equation}
X_{e+1}(0)=X_e(1)
\end{equation}
thus, while values $X_0(0)$ are sampled from the distribution $\mu_0$, values $X_e(0)$ are then determined for $e>0$. We illustrate this in figure~\ref{fig:alternative}. Now formulas~\eqref{eq:rho0def}, \eqref{eq:hdef} in the main text 
will hold, and the recursion will proceed as detailed in theorem~\ref{thm:naive_updating_behaviour} in the main text

\begin{figure}[h]
\centering
\includegraphics[width=0.75\textwidth]{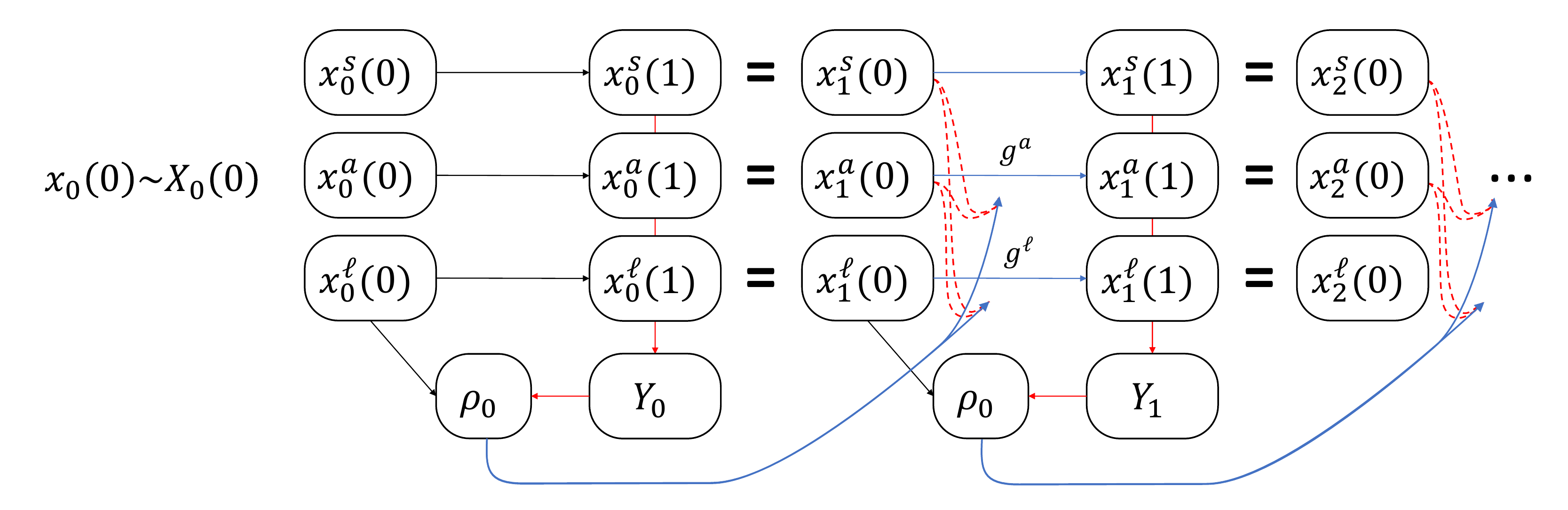}
\caption{Diagram showing alternative setup for naive updating. Values $x^s,x^a,x^{\ell}$ are sampled at $(e,t)=(0,0)$, and used to determine $\rho_0$. Values are conserved until $t=1$, and remain the same at the start of epoch 1 ($(e,t)=(1,0)$). Values are intervened on by $g^a$, $g^{\ell}$ according to $\rho_0 \left(x^s_1(0),x^a_1(0))\right)$, and resultant values at $(e,t)=(1,1)$ are conserved until the start of the next epoch at $(e,t)=(2,0)$. Lowercase leters indicates that, while quantities random-valued, they inherit all randomness from their values at $(e,t)=(0,0)$. Colour and line conventions are as for figure~\ref{fig:diagram_setup} in the main text}
\label{fig:alternative}
\end{figure}

\section{Proofs and counterexamples}
\label{supp_sec:proofs}

\subsection{Optimising both $\rho$ and $g^a$, $g^{\ell}$ is equivalent to a general resource allocation problem}
\label{supp_sec:optimiseboth}

Consider the constrained optimisation problem in section~\ref{sec:aim} in the main text. We show that if we allow $\rho$ and $g^a$, $g^{\ell}$ to vary independently, then the constrained optimisation is equivalent to the solution of a problem in which the use of a predictive score is redundant.

\begin{theorem}
Suppose that the triple 
$(\rho_{opt},g^a_{opt},g^{\ell}_{opt})$
minimises quantity~\eqref{eq:minimisethis} subject to constraint~\eqref{eq:subjecttothis} in section~\ref{sec:aim} in the main text, where all are arbitrary functions of two variables in the appropriate range. Let $h^a_{opt}$ and $h^{\ell}_{opt}$ be solutions to a second constrained optimisation problem: find $h^a(x^s,x^a)$ and $h^{\ell}(x^s,x^a,x^{\ell})$ which minimise
\begin{align}
\mathbb{E}_{X_e(0)}\{&f(X^s, \nonumber \\
&h^a(X^s_e(0),X^a_e(0)), \nonumber \\
&h^{\ell}(X^s_e(0),X^a_e(0),X^{\ell}_e(0)))\} \label{eq:altminimise}
\end{align}
subject to
\begin{align}
\mathbb{E}_{X_e(0)}\{ c^a(&X^a_e(0),   \nonumber \\
&X^a_e(0)-h^a(X^s_e(0),X^a_e(0))) +   \nonumber \\
c^{\ell}(&X^{\ell}_e(0), \nonumber \\
&X^{\ell}_e(0)-h^{\ell}(X^s_e(0),X^a_e(0),X^{\ell}_e(0)))\} \leq C & \label{eq:altsubject}
\end{align}
with $c^a,c^{\ell},f$ as for section~\ref{sec:aim}.

Then the minima of quantity~\eqref{eq:minimisethis} in the main text and of quantity~\eqref{eq:altminimise} achieved by $(\rho_{opt},g^a_{opt},g^l_{opt})$ and $(h^a_{opt},h^{\ell}_{opt})$ are the same.

\end{theorem}

\begin{proof}
Given a tuple $(\rho_{opt},g^a_{opt},g^l_{opt})$, we explicitly construct an $(h^a_{opt},h^{\ell}_{opt})$ which attains the same minimum, and vice versa.

Given $(\rho_{opt},g^a_{opt},g^l_{opt})$, the corresponding forms of $h^a_{opt}$, $h^{\ell}_{opt}$ are simply
\begin{align}
h^a_{opt}(x^s,x^a) &= g^a_{opt}\left(\rho(x^s,x^a),x^a\right) \nonumber \\
h^{\ell}_{opt}(x^s,x^a,x^{\ell}) &= g^{\ell}_{opt}\left(\rho(x^s,x^a),x^{\ell}\right) \nonumber \\
\end{align}

Given $h^a_{opt}$, $h^{\ell}_{opt}$, the correspondence is slightly more complex. Set $\rho_{opt}$ as a bijective function from $\mathbb{R}^{n_s+n_a}$ to $\mathbb{R}$; for instance, set it to `splice' the decimal digits of arguments together. Now set $g^a_{opt}$, $g^{\ell}_{opt}$ to firstly `decrypt' the value of $\rho_{opt}$ back into constituent parts ($x^s$ and $x^a$), and then compute $h^a_{opt}(x^s,x^a)$ and $h^{\ell}_{opt}(x^s,x^a,x^{\ell})$ as outputs.

This shows that the two constrained optimisation problems are equivalent.
\end{proof}

We note that this implies that optimising $(\rho,g^a,g^{\ell})$ jointly is equivalent to a more general treatment-allocation problem which does not involve a predictive score.

\subsection{Counterexample showing naive updating can cause better models to appear worse}
\label{supp_sec:models_worse}

For this counterexample we shall use the following set up:
\begin{align}
    f(x^s,x^a,x^\ell) =& f(x^s,x^a) = (1+e^{-x^s-x^a})^{-1}\\
    \rho_{0}(x^s,x^a \mid X^\star_{0},Y^\star_0) =&
    \begin{cases}
    \frac{\sum_{i=1}^{n}(Y^\star_0)_{i}\mathds{1}\{\sum_{j=1}^{2}(X^\star_{0})_{ij}>0\}}{\sum_{i=1}^{n}\mathds{1}\{\sum_{j=1}^{2}(X^\star_{0})_{ij}>0\}} & x^s + x^a > 0 \\
    \frac{\sum_{i=1}^{n}(Y^\star_0)_{i}\mathds{1}\{\sum_{j=1}^{2}(X^\star_{0})_{ij} \leq 0\}}{\sum_{i=1}^{n}\mathds{1}\{\sum_{j=1}^{2}(X^\star_{0})_{ij} \leq 0\}} & x^s + x^a \leq 0
    \end{cases}\\
    \rho_{1}(x^s,x^a \mid X^\star_{1},Y^\star_1) =&  (1+e^{-\hat{\beta}_{0} - x^s\hat{\beta}_{1} - x^a\hat{\beta}_{2}})^{-1} \mbox{ where } \hat{\beta} = \mbox{argmax}\{\mathcal{L}(\beta|X^\star_{1},Y^\star_1)\}\\
    m_{\tilde{f}_e}(\rho_{e}|X^\star_{e},Y^\star_e) =& \mathbb{E}_{\mu}\left[|f(X^s,g^a(\rho_{e-1},X^a))-\rho_e(X^s,X^a \mid X^\star_{e},Y^\star_e)|\right]\\
    g^{a}(\rho,x^a) =& (1-\rho)(x^a+3) + \rho(x^a-3)
\end{align}
For simplicity, we shall view the latent variables as having no effect on the true risk score $f$, which corresponds to the scenario where (if no interventions are made), it is possible with the data we observe to fully specify $f$. For the purpose of the counterexample it is reasonable to do this as model performance only requires $m_{\tilde{f}_e}$, which has no dependence on latent covariates.

We also state, that due to the omission of latent covariates, $X_{e}(0) = (X^s_e(0),X^a_e(0)) \sim N_{2}(0,I_{2})$, which is then used to generate (through the statistical program R) an initial training data set at epoch 0, of size $n = 100$, which is summarised below:

\begin{center}
 \begin{tabular}{| c |c | c | c|} 
 \hline
 index & $\mathbf{(X^{\star}_0)_{\cdot 1}}$ & $\mathbf{(X^{\star}_0)_{\cdot 2}}$ & $\mathbf{Y^\star_0}$ \\
 \hline
 1 & 1.185 & 1.272 & 1 \\ 
 \hline
 2 & 0.881 & -0.995 & 0 \\
 \hline
 3 & 0.122 & -0.956 & 0 \\
 \hline
 \multicolumn{1}{c}{}&\multicolumn{2}{c}{$\vdots$}\\
 \hline
 98 & -0.826 & 1.779 & 1 \\
 \hline
 99 & 0.853 & 0.151 & 1 \\
 \hline
 100 & 0.177 & 0.805 & 1 \\
 \hline
\end{tabular}
\end{center}

This training data can then inputted into $\rho_0$ to give the following function:
\begin{equation}\label{eq:p0initial}
    \rho_{0}(x^s,x^a \mid X^\star_{0},Y^\star_0) =
    \begin{cases}
    0.733 & x^s + x^a > 0 \\
    0.200 & x^s + x^a \leq 0
    \end{cases}
\end{equation}
When intervening on any covariates at epoch 1 the function given in equation \eqref{eq:p0initial} will be used to produce $X_1(1)$ and subsequently $Y_1$. 

We now consider $\mathbb{E}_{(X_0^\star,Y_0^\star)}\left[m_{\tilde{f}_0}(\rho_{0}|X^\star_0,Y^\star_0)\right]$, which we approximate using a Monte Carlo estimate with 1000 samples. However, $m_{\tilde{f}_0}(\rho_{0}|X^\star_{0},Y^\star_0)$ also requires approximation, and so a Monte Carlo estimate with the same number of samples is also used for this function. The procedure is as follows:
\begin{enumerate}
    \item For i from 1 to 1000:
    \begin{enumerate}
        \item Obtain a training data set , $(X^{\star}_0,Y^{\star}_0)_i$, by taking $n$ samples of $(X_0(0),Y_0)$.
        \item Use this training data set to obtain a $(\rho_0)_i$ of the form given in equation \eqref{eq:p0initial}.
        \item For j from 1 to 1000:
        \begin{enumerate}
            \item Sample $(x^s,x^a)_j \sim X_0(0)$.
        \end{enumerate}
        \item $m_{\tilde{f}_0}(\rho_{0}|(X^\star_{0},Y^\star_0)_i) \approx \frac{1}{1000}\sum_{j=1}^{1000}|f((x^s,x^a)_j) - \rho_0((x^s,x^a)_j \mid (X^\star_{0},Y^\star_0)_i)|$
    \end{enumerate}
    \item $\mathbb{E}_{(X_0^\star,Y_0^\star)}\left[m_{\tilde{f}_0}(\rho_{0}|X^\star_0,Y^\star_0)\right] \approx \frac{1}{1000}\sum_{j=1}^{1000}m_{\tilde{f}_0}(\rho_{0}|(X^\star_{0},Y^\star_0)_i)$
\end{enumerate}
With this in mind, we give the following approximation: $\mathbb{E}_{(X_0^\star,Y_0^\star)}\left[m_{\tilde{f}_0}(\rho_{0}|X^\star_0,Y^\star_0)\right] \approx 0.124$.

If we assert that interventions never take place, then we can use the same procedure described above to obtain $\mathbb{E}_{(X_0^\star,Y_0^\star)}\left[m_{\tilde{f}_0}(\rho_{1}|X^\star_0,Y^\star_0)\right] \approx 0.056$. So here we can clearly see that in the setting where interventions are never made, $\mathbb{E}_{(X_0^\star,Y_0^\star)}\left[m_{\tilde{f}_0}(\rho_{0}|X^\star_0,Y^\star_0)\right] > \mathbb{E}_{(X_0^\star,Y_0^\star)}\left[m_{\tilde{f}_0}(\rho_{1}|X^\star_0,Y^\star_0)\right]$, and so the model closer to the truth is the logistic regression model at epoch 1. If agents were allowed to make interventions (based on (\eqref{eq:p0initial})) however, we would consider $\mathbb{E}_{(X_1^\star,Y_1^\star)}\left[m_{\tilde{f}_1}(\rho_{1}|X^\star_1,Y^\star_1)\right] \approx 0.197$ instead. Now, since $\mathbb{E}_{(X_0^\star,Y_0^\star)}\left[m_{\tilde{f}_0}(\rho_{0}|X^\star_0,Y^\star_0)\right] < \mathbb{E}_{(X_1^\star,Y_1^\star)}\left[m_{\tilde{f}_1}(\rho_{1}|X^\star_1,Y^\star_1)\right]$, we would come to the incorrect conclusion that the model closer to the truth is the model used at epoch 1. Consequently we can state that, given the setup provided in section 3.1, 
\begin{align}
    &\mathbb{E}_{(X_0^\star,Y_0^\star)}\left[m_{\tilde{f}_0}(\rho_{0}|X^\star_0,Y^\star_0)\right] > \mathbb{E}_{(X_0^\star,Y_0^\star)}\left[m_{\tilde{f}_0}(\rho_{1}|X^\star_0,Y^\star_0)\right] \centernot\implies \nonumber\\
    &\mathbb{E}_{(X_0^\star,Y_0^\star)}\left[m_{\tilde{f}_0}(\rho_{0}|X^\star_0,Y^\star_0)\right] > \mathbb{E}_{(X_1^\star,Y_1^\star)}\left[m_{\tilde{f}_1}(\rho_{1}|X^\star_1,Y^\star_1)\right]
\end{align}

Additionally, we show that for this example:
\begin{align}
    &\mathbb{E}_{(X_0^\star,Y_0^\star)}\left[m_{\tilde{f}_0}(\rho_{0}|X^\star_0,Y^\star_0)\right] > \mathbb{E}_{(X_0^\star,Y_0^\star)}\left[m_{\tilde{f}_0}(\rho_{1}|X^\star_0,Y^\star_0)\right] \centernot\implies \nonumber\\
    &\mathbb{E}_{(X_0^\star,Y_0^\star)}\left[m_{\tilde{f}_0}(\rho_{0}|X^\star_0,Y^\star_0)\right] > \mathbb{E}_{(X_1^\star,Y_1^\star)}\left[m_{\tilde{f}_0}(\rho_{1}|X^\star_1,Y^\star_1)\right] \label{eq:estimatableinequalities}
\end{align}
as $\mathbb{E}_{(X_1^\star,Y_1^\star)}\left[m_{\tilde{f}_0}(\rho_{1}|X^\star_1,Y^\star_1)\right] \approx 0.215 > 0.124 \approx \mathbb{E}_{(X_0^\star,Y_0^\star)}\left[m_{\tilde{f}_0}(\rho_{0}|X^\star_0,Y^\star_0)\right]$. This statement is given here because for $\tilde{f}_0$, and therefore $m_{\tilde{f}_0}$, it is possible to gain estimates through a holdout test data set. Whilst the comparison is not between a risk score ($\rho_e$) and the function it is trying to estimate ($\tilde{f}_e$), the effect of deteriorating performance as epochs increase is still captured. Going further, it is assumed that if stakeholders were implementing naive model updating, they would assume that $\rho_e$ is estimating $\tilde{f}_0$ for all epochs as the belief is that interventions do not effect the model. Therefore, comparison with $\tilde{f}_0$ will heighten the impression to stakeholders that using an updated model structure is causing performance to deteriorate, especially for epoch 0 to epoch 1, where for this comparison $\rho_0$ is actually estimating $\tilde{f}_0$.

We expect from a stakeholders view that comparison (using estimates) between the two models at successive epochs usually leads to the inequality $m_{\tilde{f}_0}(\rho_{e-1} \mid X^\star_{e-1},Y^\star_{e-1}) < m_{\tilde{f}_0}(\rho_e \mid X^\star_e,Y^\star_e)$, and therefore the conclusion is that the new model leads to worse performance. We advise that a conclusion is only reached after further comparison is done between $m_{\tilde{f}_0}(\rho_{e-1} \mid X^\star_e,Y^\star_e)$ and $m_{\tilde{f}_0}(\rho_{e} \mid X^\star_e,Y^\star_e)$, as this gives an indication whether the drop in performance is due to the model structure or the intervention effect.

Finally, we advise caution when considering the effect of latent variables when estimating $m_{\tilde{f}_0}(\rho_{e}|X^\star_e,Y^\star_e)$. This is due to that fact that when holdout test data is used to obtain an estimate, it is an estimate of $f$ rather than an estimate of $\tilde{f}_0$. If the latent variables have a small influence on $f$ than $f \approx \tilde{f}_0$ and we can make inferences as shown above, but if latent variables have a large influence on $f$ then our comparison is not based on $m_{\tilde{f}_0}$ but instead on $m_f$. This creates a problem as now how well we perceive our model's performance can be determined largely by how well a model arbitrarily captures the latent covariate information using just the set and actionable covariates. It therefore becomes substantially more difficult to determine whether the cause of a models poor performance is due to the model, the intervention effect or insufficient data. As a general rule however, large values of $m_{\tilde{f}_0}(\rho_{0}|X^\star_0,Y^\star_0)$ should indicate that either the initial model is very poor or that there is insufficient data, but in either case careful consideration of what could possibly influence the underlying mechanism should be made before a risk score is built and given to agents, to ensure that latent variables affect the model as little as possible. 

\subsection{Proof of theorem~\ref{thm:naive_updating_behaviour}}
\label{supp_sec:thm1proof}

If $h'(z_0) \leq -1$ then the single fixed point of $h$ is unstable and $\rho_e$ cannot converge to it unless it was always equal to $z_0$. There can be no other $z$ with $h(z)=z_0$ since $h'(z)<0$ by assumption. 

Since $\rho_e \in [0,1]$ and $h'(z)<0$, $\rho_e$ must tend toward a stable oscillation between two values, or converge to a single value.

If the bounds on partial derivatives hold, then from the triangle and Cauchy-Schwarz inequalities, for $z \in R$
\begin{align}
|h'(z)| &\leq \mathbb{E}_{X^L}\left[  \sum_{i}^{p^a} |\delta^{g^a}_i \delta^{f^a}_i|  + \sum_{i}^{p^L} |\delta^{g^{\ell}}_i \delta^{f^{\ell}}_i|  \right] \nonumber \\
&= \sum_{i}^{p^a} |\delta^{g^a}_i|  \mathbb{E}_{X^{\ell}}\left[ | \delta^{f^a}_i|\right]  + \sum_{i}^{p^{\ell}} \mathbb{E}_{X^{\ell}}\left[|\delta^{g^{\ell}}_i \delta^{f^{\ell}}_i|  \right] \nonumber \\
&\leq \sqrt{\sum_{i}^{p^a} (\delta^{g^a}_i)^2 \sum_i^{p^a} \mathbb{E}_{X^{\ell}}\left[ \delta^{f^a}_i \right]^2} \nonumber \\
&\phantom{\leq} + \sqrt{\sum_{i}^{p^{\ell}} \mathbb{E}_{X^{\ell}}\left[\left(\delta^{g^{\ell}}_i \right)^2\right] \sum_{i}^{p^{\ell}} \mathbb{E}_{X^{\ell}}\left[ \left(\delta^{f^{\ell}}_i \right)^2 \right]} \nonumber \\
&\leq \sqrt{k_1 k_3} + \sqrt{k_2 k_4} < 1
\end{align}
so the map $h:\rho_e \to \rho_{e+1}$ is a contraction, and the convergence of the recurrence $\rho_e \to \rho_{e+1}$ follows from the Banach fixed-point theorem, as long as $\rho_e \in R$ for some value of $e$.

\subsection{Proof of theorem~\ref{thm:successive_adjuvancy} in main text}
\label{supp_sec:thm2proof}


\begin{proof}
The function $f^{-1}(x)$ is well-defined and one-to-one given assumptions~\ref{asm:univ}, \ref{asm:partial} from the theorem statement. Now
\begin{align}
\rho_{e+1} &= \rho_{e+1}(x^s,x^a) \nonumber \\
&= f\left(x^s,g^a_{e+1}(\rho_e,x^a)\right) \label{eq:rho_ge} \\
&= f\left(x^s,g^a(\rho_e,g^a_e(\rho_{e-1},x^a))\right) \nonumber \\
&= f\left(g^a(\rho_e,f^{-1}(\rho_e))\right) \nonumber \\
&\triangleq h_{2}(\rho_e) \label{eq:h2def}
\end{align}
and we have 
\begin{align}
h_2(\rho_{eq}) &= f\left(x^s, g^a(\rho_{eq},f^{-1}(\rho_{eq}))\right) \nonumber \\
&= f\left(x^s, f^{-1}(\rho_{eq})\right) \nonumber \\
&= \rho_{eq} \label{eq:rho_fixed_point}
\end{align}
so $\rho_{eq}$ is a fixed point of the recursion for $\rho_{e}$. It can be the only fixed point; by definition there is only one value of $\rho$ with $g(\rho,x)=x$, and hence for $\rho \neq \rho_{eq}$ we have $g^a(\rho,f^{-1}(\rho)) \neq f^{-1}(\rho_{eq})$. But from assumption~\ref{asm:fpartial} in the theorem statement this must mean that $h_2(\rho)=f(g^a(\rho,f^{-1}(\rho)) \neq f(f^{-1}(\rho))=\rho$.

Given the condition on the derivative of $h_2(\rho)$ for $\rho \in I$, the first result follows from the Banach fixed-point theorem. The second is immediate as the LHS is simply $\rho_e(x^s,x^a)$. The third follows from an inversion of equation~\eqref{eq:rho_ge}.
\end{proof}

\subsection{Counterexample showing failure of naive updating to generally solve constrained optimisation problem}
\label{supp_sec:nonoptimal}

For this counterexample, we do not need to consider latent covariates, and will assume they do not exist. 

Under the setting in section~\ref{sec:general} in the main text, if $\rho_n$ converges to $\rho_{\infty}(x^s,x^a)$ for some $x^s,x^a$ under naive updating, then we have
\begin{equation}
\rho_{\infty}(x^s,x^a)=h(\rho_{\infty}(x^s,x^a)=f(g(\rho_{\infty}(x^s,x^a),x^a),x^s) \label{eq:rhoinfinity}
\end{equation}
Suppose $x^s$ and $x^a$ each have dimension 1, and consider the example:
\begin{align}
f(x^a,x^s) &= \textrm{logit}(x^a + x^s) = \frac{1}{1+\exp\left(-(x^a + x^s)\right)} \nonumber \\
g(\rho,x^a) &= x^a - \log(1+\rho) \nonumber \\
c^a(x) &= x \nonumber
\end{align}
For a given function $\rho$, the objective and cost are, respectively
\begin{align}
\textrm{obj}\{\rho\} &= E\left\{(1+(1+\rho)\exp(-(X^s+X^a)))^{-1}\right\}\nonumber \\
\textrm{cost}\{\rho\} &= E\left\{ \log(1+\rho)\right\}
\end{align}
Using an oracle predictor of $Y|X$, as in the previous section, $\rho_n$ converges to the fixed point of the recursion $z \to f(g(z,x^a),x^s)$, which is
\begin{equation}
\rho_{\infty}(x^s,x^a) = \frac{1}{2}\left(\sqrt{\left(e^{x+y}+1\right)^2 + 4 e^{x+y}}- \left(e^{x+y}+1\right) \right)
\end{equation}
To see why this is not optimal, suppose $X^a,X^s$ have a discrete distribution taking either of the values $(0,-1)$, $(0,1)$ with probability $1/2$. Then
\begin{align}
\textrm{cost}\{\rho_{\infty}\} &= \frac{\log(2)}{2} \approx 0.346 \nonumber \\
\textrm{obj}\{\rho_{\infty}\} &= \frac{1+e}{1+e+ \sqrt{1+6e + e^2}} \approx 0.428 \nonumber
\end{align}
However, consider some $\rho_{0}$ with $\rho_0(0,-1)=0$, $\rho_0(0,1)=1$. Now 
\begin{align}
\textrm{cost}\{\rho_{0}\} &= \frac{\log(2)}{2} =  \textrm{cost}\{\rho_{\infty}\} \nonumber \\
\textrm{obj}\{\rho_0\} &= \frac{1}{2}\left(\frac{1}{1+e} + \frac{e}{2+e}\right) \approx 0.423 < \textrm{obj}\{\rho_{\infty}\}
\end{align}

\subsection{Simple example of updating leading to oscillation}
\label{supp_sec:oscillation}

Define $g(\rho,x^a)$ as above, and instead define 
\begin{equation}
f(x^a,x^s) = \textrm{logit}\left(-k(x^a+x^s)\right) \label{eq:fdef1}
\end{equation}
As usual, we presume that to estimate $\rho$, we regress $Y$ on $X^s_0$, $X^a_0$, and we do it accurately enough to presume $\rho$ is an oracle. Now
\begin{align}
h(x) &= \frac{1}{1+ (1+x)^k \exp \left(-k(x^s+x^a)\right)} \nonumber \\
h'(x) &= -k\frac{e^{k(x^s+x^a)}(1+x)^{k-1}}{\left( e^{k(x^s+x^a)}+(1+x)^k \right)^2}
\end{align}
Consider a setting when $x^s=x^a=0$ and $k=8$. Now $h(0)=1/2>0$ and $h(1/5) \approx 0.189 < 1/5$. For $x \in (0,1)$ we have $h'(x)<0$, so the equation $h(x)=x$ has a single solution in $(0,1/5)$. But on $(0,1/5)$, we have $h'(x)< -1$. So if $x_0$ is the unique root of $h(x)-x$ on $x \in (0,1)$ then $h'(x_0)<0$

Now as long as $\rho_0(x^s,x^a)$ is not exactly the value of $x$ for which $h(x)=x$, if we update $\rho_n$ using $h$, it can never converge as the fixed point of the map $h$ is unstable.

Conceptually, although no intervention changes $x^a$ very much, the function $f$ is very sensitive to small changes in $x^a$ when $k=8$, so a small change in $x^a$ will necessarily cause a larger change in $f(x^a,x^s)$ when $\rho$ is near the fixed point of $h$.

\section{Comparison of solution/avoidance strategies}
\label{supp_sec:solution_comparison}

We briefly compare advantages and disadvantages of the general strategies identified in section~\ref{sec:solution} to avoid or overcome problems associated with naive updating.

Any of the three strategies can be used to avoid the naive updating problem if they enable an unbiased estimate of 
\begin{equation}
\mathbb{E} \left[f_e\left(x^s,x^a,X^{\ell}\right)\right] \label{eq:critical_quantity}
\end{equation}
to be obtained, where the expectation is over $X^{\ell}$ either before or after intervention. The expectation~\eqref{eq:critical_quantity} can be recognised as the quantity for which $\rho_e$ is treated as an estimator. More frequent covariate observation as per section~\ref{sec:solution_modelling} allows this by enabling observation of $X_e(1)$, so such an unbiased estimate may be obtained by regression of $Y_e$ on observed $X_e(1)$. The strategy in  section~\ref{sec:solution_holdout} defines a hold-out subset of samples $X_e^{\star}$, $Y_e^{\star}$ for which $X_e^{\star}(1)=X_e^{\star}(0)$, so an unbiased estimate of~\eqref{eq:critical_quantity} can be obtained by regression of $Y_e^{\star}$ on (observed) $X_{e}^{\star}(0)$ will work. Finally, the strategy in section~\ref{sec:solution_control} specifies $g_e^a$ and $g_e^{\ell}$, so an unbiased estimate of~\eqref{eq:critical_quantity} can be made by regressing $Y_e$ on $X_e^S(0)$, $g_e^a(\rho_e,X_e^a(0))$.

Although all three solutions avoid the problems of naive updating, they `solve' somewhat different problems and require different experimental designs. The class of strategies described in section~\ref{sec:solution_modelling} (a range of modelling approaches generally requiring more frequent covariate observation) can solve the constrained optimisation problem in section~\ref{sec:aim} over $\rho$. The strategy described in section~\ref{sec:solution_holdout} (retention of a `hold-out' set on which no interventions are made) simply enables unbiased observation of $f_e$. The strategy described in section~\ref{sec:solution_control} (explicit control of interventions $g^a$, $g^{\ell}$) solves the constrained optimisation problem over $g^a$, $g^{\ell}$. 

However, solutions may be quantitatively compared with an aim of recommending which (if any) might be most appropriate in a given circumstance. If possible, the strategy in section~\ref{sec:solution_modelling} should be used if possible, as it enables the greatest flexibility in approach. The strategy in~\ref{sec:solution_control} should be used alternatively or additionally if appropriate.

The strategy in section~\ref{sec:solution_holdout} is advisable as a general approach if covariates cannot be observed more frequently and interventions cannot be controlled (that is, neither of the other strategies are actionable). 

\subsection{Illustration of solutions}
\label{supp_sec:solution_illustration}

We consider how each strategy may appear in the context of the setting described in Supplementary section~\ref{supp_sec:realistic_exposition}. 

The strategy in section~\ref{sec:solution_modelling} would comprise re-observing covariates in February ($t=1$) after interventions are made. Under this closer observation (allowing inference of $g^a$ and $\mathbb{E}(f)$), $\rho_e$ could be set so as to optimise healthcare provision. 

The strategy in section~\ref{sec:solution_holdout} would require nomination a random sample of the population on which scores would not be calculated, and hence on which no intervention could be made on the basis of a risk score. This would enable observation of `native' covariate effects on risk.

The strategy in section~\ref{sec:solution_control} would implement specific interventions: for instance, `if $\rho_e>50\%$, stop drug $X$'. Interventions could then be tuned to optimise healthcare provision.

\section{Open problems}
\label{supp_sec:open}

We propose the following short list of open problems in this area.
\begin{enumerate}
\item Determine a framework to modulate both $g^{\ell}$ and $g^a$ with the aim of solving the constrained optimisation problem in section~\ref{sec:aim} in the main text. 
\item Determine the dynamics and consequences of other  model-updating strategies. What happens if training data is aggregated at each step, rather than only the most recent data being used?
\item Derive results of successive adjuvancy in more general circumstances.
\item How do the dynamics of the model change when assumptions differ? Can $f$, $g^{\ell}$ and $g^a$ be extended to be random-valued, and possibly agglomerated into a single intervention function?
\item How can assumptions be changed to approximate more general machine learning settings?
\end{enumerate}

\bibliographystyle{plain}
\bibliography{references}

\makeatletter\@input{xx.tex}\makeatother